\newcount\Comments
\documentclass{article}

\usepackage{color}
\definecolor{darkgreen}{rgb}{0,0.5,0}
\definecolor{purple}{rgb}{1,0,1}
\newcommand{\kibitz}[2]{\ifnum\Comments=1\textcolor{#1}{#2}\fi}

\usepackage[english]{babel}

\usepackage[letterpaper,top=2cm,bottom=2cm,left=3cm,right=3cm,marginparwidth=1.75cm]{geometry}

\usepackage{amsmath}
\usepackage{graphicx}
\usepackage[colorlinks=true, allcolors=blue]{hyperref}

\usepackage{dirtytalk}
\usepackage{mathtools}
\usepackage[linesnumbered,ruled,vlined]{algorithm2e}

\usepackage{amsmath}
\usepackage{amsthm}
\usepackage{bbm}
\usepackage{amssymb}

\DeclareMathOperator*{\argmin}{arg\,min}
\DeclareMathOperator*{\E}{\mathbb{E}}

\DeclareMathOperator*{\one}{\mathbbm{1}}
\DeclareMathOperator*{\U}{\mathcal{U}}
\DeclareMathOperator*{\G}{\mathcal{G}}
\DeclareMathOperator*{\X}{\mathcal{X}}

\usepackage{mathtools}

\newtheorem{theorem}{Theorem}
\newtheorem{proposition}[theorem]{Proposition}
\newtheorem{definition}{Definition}
\newtheorem{lemma}[theorem]{Lemma}

\theoremstyle{remark}

\title{On Proper Learnability between Average- and Worst-case Robustness}
\author{Vinod Raman, Unique Subedi, Ambuj Tewari}
\usepackage{natbib}

\begin{document}

\maketitle

\begin{abstract}
Recently, \citet{montasser2019vc} showed that finite VC dimension is not sufficient for \textit{proper} adversarially robust PAC learning. In light of this hardness, there is a growing effort to study what type of relaxations to the adversarially robust PAC learning setup can enable proper learnability. In this work, we initiate the study of proper learning under relaxations of the worst-case robust loss. We give a family of robust loss relaxations under which VC classes are properly PAC learnable with sample complexity close to what one would require in the standard PAC learning setup. On the other hand, we show that for an existing and natural relaxation of the worst-case robust loss, finite VC dimension is not sufficient for proper learning. Lastly, we give new generalization guarantees for the adversarially robust empirical risk minimizer.%
\end{abstract}

\section{Introduction}

 As deep neural networks become increasingly ubiquitous, their susceptibility to test-time adversarial attacks has become more and more apparent.  Designing learning algorithms that are robust to these test-time adversarial perturbations has garnered increasing attention by machine learning researchers and practitioners alike.  Prior work on adversarially robust learning has mainly focused on learnability under the \textit{worst-case} robust risk \citep{montasser2019vc, attias2021improved, cullina2018pac},
$$R_{\U}(h;\mathcal{D}) := \E_{(x, y) \sim \mathcal{D}} \left[\sup_{z \in \U(x)}\one\{h(z) \neq y \} \right],$$
where  $\U(x) \subset \mathcal{X}$ is an arbitrary but fixed perturbation set (for example $\ell_p$ balls). In practice, worst-case adversarial robustness is commonly achieved via Empirical Risk Minimization (ERM) of the robust loss or some convex surrogate \citep{madry2017towards, wong2018provable, raghunathan2018certified, bao2020calibrated}. However, a seminal result by \cite{montasser2019vc} shows that any proper learning rule, including ERM, even when trained on an arbitrarily large number of samples, may not return a classifier with small robust risk. These high generalization gaps for the robust loss have also been observed in practice \citep{schmidt2018adversarially}. 
Even worse, empirical studies have shown that classifiers trained to achieve worst-case adversarial robustness exhibit degraded \textit{nominal} performance \citep{dobriban2020provable,raghunathan2019adversarial,su2018robustness,tsipras2018robustness,yang2020closer,zhang2019theoretically,robey2022probabilistically}. 

In light of these difficulties, there has been a recent push to study when proper learning, and more specifically, when learning via ERM is possible for achieving adversarial robustness. The ability to achieve test-time robustness via proper learning rules is important from a practical standpoint. It aligns better with the current approaches used in practice (e.g. (S)GD-trained deep nets),  
and proper learning algorithms are often simpler to implement than improper ones. In this vain, \cite{ashtiani2022adversarially} and \cite{bhattacharjee2022robust} consider adversarial robust learning in the \textit{tolerant} setting,  where the error of the learner is compared with the best achievable error w.r.t. a slightly \textit{larger} perturbation set. They show that the sample complexity of tolerant robust learning can be significantly lower than the current known sample complexity for adversarially robust learning and that proper learning via ERM can be possible under certain assumptions. Additionally, \cite{pmlr-v119-ashtiani20a} studied the proper learnability of VC classes under a PAC-type framework of semisupervised learning. In a different direction, several works have considered relaxing the worst-case nature of the adversarial robust loss $\ell_{\mathcal{U}}(h, (x, y)) = \sup_{z \in \mathcal{U}(x)}\mathbbm{1}\{h(z) \neq y \}$ \citep{robey2022probabilistically,li2020tilted,li2021tilted,laidlaw2019functional,rice2021robustness}. However, the PAC learnability of these relaxed notions of adversarial robust loss has not been well studied. 

In this paper, we study relaxations of the worst-case adversarially robust learning setup from a learning-theoretic standpoint.  We classify existing relaxations of worst-case adversarial robust learning into two approaches: one based on relaxing the loss function and the other based on relaxing the benchmark competitor. Much of the existing learning-theoretic work studying relaxations of adversarial robustness focus on the latter approach. These works answer the question of whether proper PAC learning is feasible if the learner is evaluated against a stronger notion of robustness. In contrast, we focus on the \textit{former} relaxation and pose the question: \emph{can proper PAC learning be feasible if we relax the adversarial robust loss function itself?}
In answering this question, we make the following main contributions:
\begin{itemize} 
\itemsep0em

\item  We show that the finiteness of the VC dimension is {\em not sufficient} for properly learning a natural robust loss relaxation proposed by \citet{robey2022probabilistically}. Our proof techniques involve constructing a VC class that is not properly learnable. 

\item We give a family of robust loss relaxations that interpolate between average- and worst-case robustness. For these losses, we use Rademacher complexity arguments relying on the Ledoux-Talagrand contraction to show that all VC classes are learnable via ERM.

\item We extend a property implicitly appearing in margin theory (e.g., see \citet[Section 5.4]{mohri2018foundations}), which we term ``Sandwich Uniform Convergence'' (SUC), to show new generalization guarantees for the adversarially robust empirical risk minimizer. 

\end{itemize}
%

\section{Preliminaries and Notation}
Throughout this paper we let $[k]$ denote the set of integers $\{1, ..., k\}$, $\mathcal{X}$ denote an instance space, $\mathcal{Y} = \{-1, 1\}$ denote our label space, and $\mathcal{D}$ be any distribution over $\mathcal{X} \times \mathcal{Y}$. Let $\mathcal{H} \subset \mathcal{Y}^{\mathcal{X}}$ denote a hypothesis class mapping examples in $\mathcal{X}$ to labels in $\mathcal{Y}$. 

\subsection{Problem Setting}

In the standard robust learning setting, there exists an adversary who picks an arbitrary index set $\mathcal{G}$ of perturbation functions $g: \X \rightarrow \X$. At test time, the adversary intercepts the labeled example $(x, y)$, exhaustively searches over the perturbation set to find \emph{the worst} perturbation function $g$, and then passes the perturbed instance $g(x)$ to the learner. From this perspective, the adversarially robust loss is defined as $\ell_{\G}(h, (x, y)) := \sup_{g \in \G}\mathbbm{1}\{h(g(x)) \neq y\}$ and its corresponding risk as
$R_{\G}(h;\mathcal{D}) = \E_{(x, y) \sim \mathcal{D}} \left[\ell_{\G}(h, (x, y)) \right].$  We highlight that our use of perturbation functions $\mathcal{G}$ instead of perturbation sets $\mathcal{U}$ is without loss of generality (see Appendix \ref{app:eq} for an equivalence).

However, such a worst-case adversary may be too strong and unnatural, especially in high-dimension. Accordingly, we relax this model by considering a \textit{lazy} adversary that picks both a perturbation set $\mathcal{G}$ and a measure $\mu$ over $\mathcal{G}$. At test-time, the lazy adversary intercepts the labeled example $(x, y)$, \textit{randomly samples} a perturbation function $g \sim \mu$, and then passes the perturbed instance $g(x)$ to the learner. From this interpretation, the goal of the learner is to output a hypothesis such that the \textit{probability} that the lazy adversary succeeds in sampling a bad perturbation function, for any labeled example in the support of $\mathcal{D}$, is small. 


To capture this probabilistic relaxation of worst-case robustness, we consider losses that 
%
are a function of $\mathbb{P}_{g \sim \mu}[h(g(x)) \neq y]$. For a labelled example $(x, y) \in \mathcal{X} \times \mathcal{Y}$, $\mathbbm{P}_{g \sim \mu}\left[h(g(x)) \neq y\right]$ measures the fraction of perturbations in $\mathcal{G}$ for which the classifier $h$ is non-robust. Observe that $\mathbbm{P}_{g \sim \mu}\left[h(g(x)) \neq y \right] = \frac{1 - y\mathbbm{E}_{g \sim \mu}\left[h(g(x)) \right]}{2}$ is an affine transformation of  quantity $y\mathbbm{E}_{g \sim \mu}\left[h(g(x)) \right]$, the probabilistically robust \textit{margin} of $h$ on $(x, y)$ w.r.t. $(\mathcal{G}, \mu)$.  Thus, we focus on loss functions that operate over the margin $y\mathbbm{E}_{g \sim \mu}\left[h(g(x)) \right]$. 

In this work, we are primarily interested in understanding whether probabilistic relaxations of the worst-case robust loss enable ERM-based (proper) learning. That is, given a hypothesis class $\mathcal{H}$, adversary $(\mathcal{G}, \mu)$, loss function $\ell_{\mathcal{G}, \mu}(h, (x, y)) = \ell(y\mathbbm{E}_{g \sim \mu}\left[h(g(x)) \right])$,  and labelled samples from an unknown distribution  $\mathcal{D}$, our goal is to design a \textit{proper} learning algorithm $\mathcal{A}: (\mathcal{X} \times \mathcal{Y})^* \rightarrow \mathcal{H}$ such that for any distribution $\mathcal{D}$ over $\mathcal{X} \times \mathcal{Y}$, the algorithm $\mathcal{A}$ finds a hypothesis $h \in \mathcal{H}$ with low risk with regards to $\ell_{\mathcal{G}, \mu}(h, (x, y))$. 

\subsection{Complexity Measures}

Under the standard 0-1 risk, the Vapnik-Chervonenkis dimension (VC dimension)  plays an important role in characterizing PAC learnability, and more specifically, when ERM is possible. A hypothesis class $\mathcal{H}$ is PAC learnable if and only if its VC dimension is finite \citep{vapnik1971uniform}.
\begin{definition}[VC Dimension]
A set $\{x_1, ..., x_n\} \in \mathcal{X}$ is shattered by $\mathcal{H}$, if $\forall y_1, ..., y_n \in \mathcal{Y}$, $\exists h\in \mathcal{H}$, s.t. $\forall i \in [n]$, $h(x_i) = y_i$. The VC dimension of $\mathcal{H}$, denoted $\text{VC}(\mathcal{H})$, is defined as the largest natural number $n \in \mathbbm{N}$ such that there exists a set $\{x_1, ..., x_n\} \in \mathcal{X}$ that is shattered by $\mathcal{H}$.
\end{definition}

 One \textit{sufficient} condition for proper, ERM-based learning, based on Vapnik's ``General Learning" \citep{vapnik2006estimation}, is the finiteness of the VC dimension of a binary loss class  
 $$\mathcal{L}^{\mathcal{H}} := \{(x, y) \mapsto \ell(h, (x, y)): h \in \mathcal{H}\}$$
 where $\ell(h, (x, y))$ is some loss function mapping to $\{0, 1\}$. In particular, if the VC dimension of the  loss class $\mathcal{L}^{\mathcal{H}}$ is finite, then $\mathcal{H}$ is PAC learnable via oracle access to an ERM for $\ell$ with sample complexity that scales linearly with $\text{VC}(\mathcal{L}^{\mathcal{H}})$. In this sense, if one can upper bound $\text{VC}(\mathcal{L}^{\mathcal{H}})$ in terms of $\text{VC}(\mathcal{H})$, then finite VC dimension is sufficient for proper, ERM-based, learnability. Unfortunately, for adversarially robust learning, \cite{montasser2019vc} show that there can be an arbitrary gap between the VC dimension of the adversarially robust loss class $\mathcal{L}_{\mathcal{G}}^{\mathcal{H}} := \{(x, y) \mapsto \ell_{\mathcal{G}}(h, (x, y)): h \in \mathcal{H}\}$
and the VC dimension of $\mathcal{H}$. Likewise, in Section \ref{sec:relaxedlossnothelp}, we show that for some natural relaxations of the adversarial robust loss, there can also be an arbitrarily large gap between the VC dimension of the loss class and the VC dimension of the hypothesis class. 

As many of the loss functions we consider will actually map to values in $\mathbbm{R}$,  the VC dimension of the loss class will not be well defined. Instead, we can capture the complexity of the loss class via the \textit{empirical} Rademacher complexity.

\begin{definition} [Empirical Rademacher Complexity of Loss Class] 

Let $\ell$ be a loss function, $S = \{(x_1, y_1), ..., (x_n, y_n)\} \in (\mathcal{X} \times \mathcal{Y})^*$ be a set of examples, and $\mathcal{F} = \{(x, y) \mapsto \ell(h,(x,y)): h \in \mathcal{H}\}$ be a loss class. The empirical Rademacher complexity of $\mathcal{F}$ is defined as 
$$\hat{\mathfrak{R}}_m(\mathcal{F}) = \mathbbm{E}_{\sigma}\left[\sup_{f \in \mathcal{F}} \left(\frac{1}{n}\sum_{i=1}^m \sigma_i f(x_i, y_i) \right)\right]$$
where $\sigma_1, ..., \sigma_m$ are independent \emph{Rademacher} random variables. 
\end{definition}

A standard result relates the empirical Rademacher complexity to the generalization error of hypotheses in $\mathcal{H}$ w.r.t. a real-valued bounded loss function $\ell(h, (x, y))$ \citep{bartlett2002rademacher}.

\begin{proposition}[Rademacher-based Uniform Convergence]
\label{thm:rad}
  Let $\mathcal{D}$ be a distribution over $\mathcal{X} \times \mathcal{Y}$ and $\ell(h, (x, y)) \leq c$ be a bounded loss function. With probability at least $1 - \delta$ over the sample $S \sim \mathcal{D}^m$, for all $h \in \mathcal{H}$ simultaneously, 
$$\left|\mathbbm{E}_\mathcal{D}[\ell(h(x), y)] - \hat{\mathbbm{E}}_S[\ell(h(x), y)]\right| \leq  2\hat{\mathfrak{R}}_m(\mathcal{F}) + O\left(c\sqrt{\frac{\ln(\frac{1}{\delta})}{n}}\right)$$
where $\hat{\mathbbm{E}}_S[\ell(h(x), y)] = \frac{1}{|S|}\sum_{(x, y) \in S} \ell(h(x), y)$ is the empirical average of the loss over $S$.
\end{proposition}

\section{Not All Robust Loss Relaxations Enable Proper Learning}
\label{sec:relaxedlossnothelp}

We begin our study of robust loss relaxations by considering the $\rho$-probabilistically robust loss, 
$$\ell^{\rho}_{\G, \mu}(h, (x, y)) := \mathbbm{1}\{\mathbbm{P}_{g \sim \mu}\left( h(g(x)) \neq y \right) > \rho\},$$
 where $\rho \in[0,1)$ is selected apriori. The $\rho$-probabilistically robust loss was first introduced by \cite{robey2022probabilistically} for the case when $\mathcal{X} = \mathbbm{R}^d$, $g_c(x) = x + c$, and the set of perturbations $\mathcal{G} = \{g_c: c \in \Delta\}$ for some $\Delta \subset \mathbbm{R}^d$. In this paper, we generalize this loss to an arbitrary instance space $\mathcal{X}$ and perturbation set $\mathcal{G}$.  Learning under the $\rho$-probabilistically robust loss asks to find a hypothesis $h \in \mathcal{H}$ that is robust to at least a $1-\rho$ fraction of the perturbations in $\mathcal{G}$ for each example in the support of the data distribution $\mathcal{D}$.  To that end, we let $R^{\rho}_{\G, \mu}(h;\mathcal{D}) = \E_{(x, y) \sim \mathcal{D}} \left[\ell^{\rho}_{\G, \mu}(h, (x, y)) \right]$ denote the $\rho$-probabilistically robust risk.
 \begin{definition}[$\rho$-Probabilistically Robust Learning]\label{def:PRPAC} For any $\epsilon, \delta \in (0, 1)$ and any $\rho \in [0, 1)$, the sample complexity of $\rho$-probabilistically robust $(\epsilon, \delta)$-learning of $\mathcal{H}$ w.r.t. adversary $(\mathcal{G}, \mu)$, denoted $n(\epsilon, \delta, \rho; \mathcal{H}, \mathcal{G}, \mu)$, is the smallest number $m \in \mathbbm{N}$ for which there exists a learning rule $\mathcal{A}: (\mathcal{X} \times \mathcal{Y})^* \rightarrow \mathcal{Y}^{\mathcal{X}}$ such that for every distribution $\mathcal{D}$ over $\mathcal{X} \times \mathcal{Y}$, with probability at least $1 - \delta$ over $S \sim \mathcal{D}^m$, 
$$R^{\rho}_{\G, \mu}(\mathcal{A}(S);\mathcal{D}) \leq \inf_{h \in \mathcal{H}}R^{\rho}_{\G,\mu}(h;\mathcal{D}) + \epsilon.$$
We say that $\mathcal{H}$ is probabilistically robustly PAC learnable  w.r.t. adversary $(\mathcal{G}, \mu)$ at a level of $\rho$, if $\forall \epsilon, \delta \in (0, 1)$, $n(\epsilon, \delta, \rho; \mathcal{H}, \mathcal{G}, \mu)$ is finite.
\end{definition}
As highlighted by \cite{robey2022probabilistically}, this notion of robustness is desirable as it nicely interpolates between worst- and average-case robustness via an interpretable parameter $\rho$, while being more computationally tractable compared to existing relaxations. 

Which hypothesis classes are probabilistically robustly learnable, and that so using proper learning rules which output predictors in $\mathcal{H}$? 
Our main result in this section, Theorem \ref{thm:propernotposs}, shows that if $\mathcal{G}$ is allowed to be arbitrary,  then VC dimension is not sufficient for \textit{proper} $\rho$-probabilistically robust PAC learning, let alone learning via ERM.

\begin{theorem}
\label{thm:propernotposs}
For every $\rho \in [0, 1)$, there exists a hypothesis class $\mathcal{H} \subset \mathcal{Y}^{\mathcal{X}}$ with $\text{VC}(\mathcal{H}) \leq 1$ and an adversary $(\mathcal{G}, \mu)$ such that  $\mathcal{H}$ is not properly $\rho$-probabilistically robustly PAC learnable. 
\end{theorem}

To prove Theorem \ref{thm:propernotposs}, we fix $\mathcal{X} = \mathbbm{R}^d$, $\mathcal{G} = \{g_{\delta}: \delta \in \mathbbm{R}^d, ||\delta||_p \leq \gamma\}$ s.t. $g_{\delta}(x) = x + \delta$ for all $x \in \mathcal{X}$ for some $\gamma > 0$, and $\mu$ to be the uniform measure over $\mathcal{G}$. In other words, we are picking our perturbation sets to be $\ell_p$ balls of radius $\gamma$ and our perturbation measures to be uniform over each perturbation set. Note that by construction of $\mathcal{G}$, a uniform measure $\mu$ over $\mathcal{G}$ also induces a uniform measure $\mu_x$ over $\G(x) := \{g_{\delta}(x): g_{\delta} \in \G\} \subset \mathbbm{R}^d$. It will be useful to define  the $\rho$-probabilistically robust loss class $\mathcal{L}_{\mathcal{G}, \mu}^{\mathcal{H}, \rho} := \{(x, y) \mapsto \one \{\ell_{\G, \mu}^{\rho}(h, (x, y)): h \in \mathcal{H}\}$. We start by showing that for every $\rho \in [0, 1)$, there can be an arbitrary gap between the VC dimension of the loss  and hypothesis class. 

 \begin{lemma}
 \label{lem1}
 For every $\rho \in [0, 1)$ and $m \in \mathbbm{N}$, there exists a hypothesis class $\mathcal{H} \subset \mathcal{Y}^{\mathcal{X}}$ s.t. $\text{VC}(\mathcal{H}) \leq 1$ but $\text{VC}(\mathcal{L}_{\mathcal{G}, \mu}^{\mathcal{H}, \rho}) \geq m$.
 \end{lemma}
\begin{proof}
Fix $\rho \in [0, 1)$ and let $m \in \mathbbm{N}$. Pick $m$ center points $c_1, ..., c_m$ in $\X$ such that for all $i, j \in [m]$, $\mathcal{G}(c_i) \cap \mathcal{G}(c_j) = \emptyset$. For each center $c_i$, consider $2^{m-1}+1$ disjoint subsets of its perturbation set $\mathcal{G}(c_i)$ which do not contain $c_i$.  Label $2^{m-1}$ of these subsets with a unique bitstring $b \in \{0, 1\}^m$ fixing  $b_i = 1$. Let $\mathcal{B}_i^b$ denote the subset labeled by bitstring $b$ and let $\mathcal{B}_i$ denote the single remaining subset that was not labeled. Furthermore, for each $i \in [m]$ and $b \in \{\{0, 1\}^m| b_i = 1\}$, pick $\mathcal{B}_i$ and $\mathcal{B}_i^b$'s s.t. $\mu_{c_i}(\mathcal{B}_i) = \rho$ and  $0 < \mu_{c_i}(\mathcal{B}^b_i) \leq \frac{1-\rho}{2^{m}}$. If $b_i = 0$, let $\mathcal{B}_{i}^b = \emptyset$. If $\rho = 0$, let $\mathcal{B}_i = \emptyset$ for all $i \in [m]$. Finally, define $\mathcal{B} = \bigcup_{i=1}^m \bigcup_{b \in \{0, 1\}^m} \mathcal{B}_i^b \cup \mathcal{B}_i$ as the union of all the subsets. Crucially, observe that for all $i \in [m]$, $\mu_{c_i}\left(\mathcal{B}_i \cup \left(\bigcup_{b} \mathcal{B}_i^b\right)\right) \leq \frac{1 + \rho}{2} < 1$.

For bitstring $b \in \{0, 1\}^m$, define the hypothesis $h_b$ as
 $$h_b(z) = \begin{cases}
0 & \text{if $z \in \bigcup_{i=1}^m \mathcal{B}_i^b$} \cup \mathcal{B}_i \\
1 & \text{otherwise}
\end{cases}$$
and consider the hypothesis class $\mathcal{H} = \{h_b | b \in \{0, 1\}^m\}$ which consists of all $2^m$ hypothesis, one for each bitstring.  We first show that $\mathcal{H}$ has VC dimension at most $1$. Consider two points $x_1, x_2 \in \mathcal{X}$. We will show case by case that every possible pair of points cannot be shattered by $\mathcal{H}$. First, consider the case where, wlog, $x_1 \notin \mathcal{B}$. Then, $\forall h \in \mathcal{H}$, $h(x_1) = 1$, and thus shattering is not possible. Now, consider the case where both $x_1 \in \mathcal{B}$ and $x_2 \in \mathcal{B}$. If either $x_1$ or $x_2$ is in $\bigcup_{i=1}^m \mathcal{B}_i$, then every hypothesis $h \in \mathcal{H}$ will label it as $0$, and thus these two points cannot be shattered. If $x_1 \in \mathcal{B}^b_i$ and $x_2 \in \mathcal{B}^b_j$ for $i \neq j$, then $h_b(x_1) = h_b(x_2) = 0$, but $\forall h \in \mathcal{H} \text{ s.t. } h \neq h_b, h(x_1) = h(x_2) = 1$. If $x_1 \in \mathcal{B}_i^{b_1}$ and $x_2 \in \mathcal{B}_j^{b_2}$ for $b_1 \neq b_2$, then there exists no hypothesis in $\mathcal{H}$ that can label $(x_1, x_2)$ as $(0, 0)$. Thus, overall, no two points $x_1, x_2 \in \mathcal{X}$ can be shattered by $\mathcal{H}$.

  Now we are ready to show that the VC dimension of the loss class is at least $m$. Specifically, given the sample of labelled points $S = \{(c_1, 1), ..., (c_m, 1)\}$, we will show that the loss behavior corresponding to hypothesis $h_b$ on the sample $S$ is exactly $b$. Since $\mathcal{H}$ contains all the hypothesis corresponding to every single bitstring $b \in \{0, 1\}^m$, the loss class of $\mathcal{H}$ will shatter $S$. In order to prove that  the loss behavior of $h_b$ on the sample $S$ is exactly $b$, it suffices to show that the probabilistic loss of $h_b$ on example $(c_i, 1)$ is $b_i$, where $b_i$ denotes the $i$th bit of $b$. By definition, 
\begin{align*}
    \ell_{\mathcal{G}, \mu}^{\rho}(h_b, (c_i, 1)) &= \one\{\mathbbm{P}_{g \sim \mu}\left( h_b(g(c_i)) \neq 1\right) > \rho\}\\
    &= \one\{\mathbbm{P}_{z \sim \mu_{c_i}}\left( h_b(z) = 0\right) > \rho\} \\
    &= \one\{\mathbbm{P}_{z \sim \mu_{c_i}}\left( z \in \mathcal{B}_i^b \cup \mathcal{B}_i \right) > \rho\}\\
    &=  \one\{\mu_{c_i}(\mathcal{B}_i^b \cup \mathcal{B}_i) > \rho\}\\ 
    &= b_i.\\
\end{align*}
Thus, the loss behavior of $h_b$ on $S$ is $b$, and the total number of distinct loss behaviors over each hypothesis in $\mathcal{H}$ on $S$ is $2^m$, implying that the VC dimension of the loss class is at least $m$. This completes the construction and proof of the claim. 
\end{proof}

We highlight two key differences between Lemma \ref{lem1} and its analog, Lemma 2, in \cite{montasser2019vc}. First, we need to provide \textit{both} a perturbation set and a perturbation measure. The interplay between these two objects is not present in \cite{montasser2019vc} and, apriori, it is not clear that these would indeed be $\ell_p$ balls and the uniform measure. Second, in order for a hypothesis to be probabilistically non-robust there needs to exist a large enough \textit{region} of perturbations over which it makes mistakes. This is in contrast to \cite{montasser2019vc}, where a hypothesis is  adversarially non-robust as long as there exists \textit{one} non-robust perturbation. 
Constructing a hypothesis class that achieves all possible probabilistically robust loss behaviors while also having low VC dimension is non-trivial -  we need hypotheses to be expressive enough to have large regions of non-robustness while not being too expressive such that VC dimension increases. 

Next, we show that the hypothesis class construction in Lemma \ref{lem1} can be used to show the existence of a hypothesis class that cannot be learned properly. Lemma \ref{lem2} is similar to Lemma 3 in \cite{montasser2019vc} and is proved in Appendix \ref{app:rhosecond}. 

\begin{lemma}
\label{lem2}
For every $\rho \in [0, 1)$ and $m \in \mathbbm{N}$ there exists $\mathcal{H} \subset \mathcal{Y}^{\mathcal{X}}$ with $\text{VC}(\mathcal{H}) \leq 1$ such that for any proper learner $\mathcal{A}: (\mathcal{X} \times \mathcal{Y})^{*} \rightarrow \mathcal{H}$: \emph{(1)} there is a distribution $\mathcal{D}$ over $\mathcal{X} \times \mathcal{Y}$ and a hypothesis $h^{*} \in \mathcal{H}$ where $R^{\rho}_{\mathcal{G}, \mu}(h^{*};\mathcal{D}) = 0$ and \emph{(2)} with probability at least $1/7$ over $S \sim D^{m}$, $R^{\rho}_{\mathcal{G}, \mu}(\mathcal{A}(S);\mathcal{D}) > 1/8$.
\end{lemma}

Finally, the proof of Theorem \ref{thm:propernotposs} uses Lemma \ref{lem2} and follows a similar idea as its analog in \cite{montasser2019vc} (Theorem 1). However, since our hypothesis class construction in Lemma \ref{lem1} is different, some subtle modifications need to be made. We include a complete proof in Appendix \ref{app:propernotposs}.

\section{Proper Learnability Under Relaxed Losses}
\label{sec:relaxedloss}

Despite the fact that VC classes are not  $\rho$-probabilistically robustly learnable using proper learning rules, in this section, we show that our framework still enables us to capture a wide range of robust loss relaxations for which proper learning is possible. 

In particular, consider robust loss relaxations of the form $\ell_{\mathcal{G}, \mu}(h, (x, y)) = \ell(y\mathbbm{E}_{g \sim \mu}\left[h(g(x))\right])$ where $\ell(t): \mathbbm{R} \rightarrow \mathbbm{R}$ is a $L$-Lipschitz function. This class of loss functions is general, capturing many natural robust loss relaxations like the hinge loss $1 - y\mathbbm{E}_{g \sim \mu}\left[h(g(x))\right]$,  squared loss $(y - \mathbbm{E}_{g \sim \mu}\left[h(g(x))\right])^2 = (1 - y\mathbbm{E}_{g \sim \mu}\left[h(g(x))\right])^2$, and exponential loss $e^{-y\mathbbm{E}_{g \sim \mu}\left[h(g(x))\right]}$.  Furthermore, the class of Lipschitz functions $\ell:\mathbbm{R} \rightarrow \mathbbm{R}$ on the margin $y\mathbbm{E}_{g \sim \mu}\left[h(g(x))\right]$ enables us to capture levels of robustness between the average- and worst-case.  For example, taking $\ell(t) = \frac{1 - t}{2}$ results in the loss $\ell_{\mathcal{G}, \mu}(h, (x, y)) = \ell(y\mathbbm{E}_{g \sim \mu}\left[h(g(x))\right]) = \mathbbm{P}_{g \sim \mu}\left[h(g(x)) \neq y\right]$, corresponding to average-case robustness, or \textit{data augmentation}. On the other hand, taking $\ell(t) = \min(\frac{1 - t}{2\rho}, 1)$ for some $\rho \in (0, 1)$, results in the loss 
$$\ell_{\mathcal{G}, \mu}(h, (x, y)) = \ell(y\mathbbm{E}_{g \sim \mu}\left[h(g(x))\right]) = \min\left(\frac{\mathbbm{P}_{g \sim \mu}\left[h(g(x)) \neq y\right]}{\rho}, 1 \right)$$ 
which corresponds to a notion of robustness that becomes stricter as $\rho$ approaches $0$. We note that some of the losses in our family were studied by \cite{rice2021robustness}. However, their focus was on evaluating robustness, while ours is about (proper) learnability. 

Lemma \ref{lem: uclip} shows that for hypothesis classes $\mathcal{H}$ with finite VC dimension, for any $(\mathcal{G}, \mu)$, all $L$-Lipschitz loss functions $\ell_{\mathcal{G}, \mu}(h, (x, y))$ enjoy the uniform convergence property.

\begin{lemma}[Uniform Convergence of Lipschitz Loss]
\label{lem: uclip}
Let $\mathcal{H}$ be a hypothesis class with finite VC dimension,  $(\G, \mu)$ be an adversary, and $\ell_{\mathcal{G}, \mu}(h, (x, y)) = \ell\left(y\mathbbm{E}_{g \sim \mu}\left[h(g(x))\right]\right)$ s.t.  $\ell: \mathbbm{R} \rightarrow \mathbbm{R}$ 
is a $L$-Lipschitz function. With probability at least $1- \delta$  over a sample $S \sim \mathcal{D}^n$ of size $n = O\left( \frac{\text{VC}({\mathcal{H}})L^2 \ln(\frac{L}{\epsilon}) + \ln(\frac{1}{\delta})}{\epsilon^2} \right)$, for all $h \in \mathcal{H}$ simultaneously, 
$$\left|\mathbbm{E}_{\mathcal{D}}\left[\ell_{\mathcal{G}, \mu}(h, (x, y)) \right] - \hat{\mathbbm{E}}_{S}\left[\ell_{\mathcal{G}, \mu}(h, (x, y)) \right]\right| \leq \epsilon.$$
\end{lemma}

\begin{proof}  
Let $\text{VC}(\mathcal{H}) = d$ and  $S = \{(x_1, y_1), ..., (x_m, y_m)\}$ be a set of examples drawn i.i.d from $\mathcal{D}$. Define the loss class $\mathcal{L}^{\mathcal{H}}_{\mathcal{G}, \mu} = \{(x, y) \mapsto \ell_{\G, \mu}(h, (x, y)): h \in \mathcal{H}\}$. Observe that we can reparameterize $\mathcal{L}^{\mathcal{H}}_{\mathcal{G}, \mu}$ as the composition of a $L$-Lipschitz function $\ell(x)$ and the function class $\mathcal{F}_{\mathcal{G}, \mu}^{\mathcal{H}} = \{(x, y) \mapsto y\mathbbm{E}_{g \sim \mu}\left[h(g(x))\right]: h \in \mathcal{H}\}$. By Proposition \ref{thm:rad}, to show the uniform convergence property of $\ell_{\G, \mu}(h, (x, y))$, it suffices to upper bound $\hat{\mathfrak{R}}_m(\mathcal{L}^{\mathcal{H}}_{\mathcal{G}, \mu}) = \hat{\mathfrak{R}}_m(\ell \circ \mathcal{F}_{\mathcal{G}, \mu}^{\mathcal{H}})$, the empirical Rademacher complexity  of the loss class. Since $\ell$ is $L$-Lipschitz, by Ledoux-Talagrand's contraction principle \citep{ledoux1991probability}, it follows that $\hat{\mathfrak{R}}_m(\mathcal{L}_{\mathcal{G}, \mu}^{\mathcal{H}}) = \hat{\mathfrak{R}}_m(\ell \circ \mathcal{F}_{\mathcal{G}, \mu}^{\mathcal{H}}) \leq L \cdot \hat{\mathfrak{R}}_m(\mathcal{F}_{\mathcal{G}, \mu}^{\mathcal{H}}).$ 
Thus, it actually suffices to upperbound $\hat{\mathfrak{R}}_m(\mathcal{F}_{\mathcal{G}, \mu}^{\mathcal{H}})$ instead. Starting with the definition of the empirical Rademacher complexity: 
\begin{align*}
\hat{\mathfrak{R}}_m(\mathcal{F}_{\mathcal{G}, \mu}^{\mathcal{H}}) &= \frac{1}{m}\mathbbm{E}_{\sigma \sim \{\pm 1\}^m}\left[\sup_{h \in \mathcal{H}} \left(\sum_{i=1}^m \sigma_i y_i\mathbbm{E}_{g \sim \mu}\left[ h(g(x_i)) \right] \right)\right]  \\
&= \frac{1}{m}\mathbbm{E}_{\sigma \sim \{\pm 1\}^m}\left[\sup_{h \in \mathcal{H}} \left(\mathbbm{E}_{g \sim \mu}\left[\sum_{i=1}^m \sigma_i  h(g(x_i)) \right] \right)\right]\\
&\leq \mathbbm{E}_{g \sim \mu} \left[\frac{1}{m}\mathbbm{E}_{\sigma \sim \{\pm 1\}^m} \left[ \sup_{h \in \mathcal{H}}\sum_{i=1}^m \sigma_i  h(g(x_i))\right] \right],
\end{align*}

where the last inequality follows from Jensen's inequality and Fubini's Theorem. 
%
Note that the quantity $\frac{1}{m}\mathbbm{E}_{\sigma \sim \{\pm 1\}^m} \left[ \sup_{h \in \mathcal{H}}\sum_{i=1}^m \sigma_i  h(g(x_i))\right]$ is the empirical Rademacher complexity of the hypothesis class $\mathcal{H}$ over the sample $\{g(x_1), ..., g(x_m)\}$ drawn i.i.d from the distribution defined by first sampling from the marginal data distribution,  $x \sim \mathcal{D}_{\mathcal{X}}$, and then applying the transformation $g(x)$. By standard VC arguments, $\hat{\mathfrak{R}}_m (\mathcal{H}) \leq O\left(\sqrt{\frac{d \ln(\frac{m}{d})}{m}}\right)$, which implies that $\hat{\mathfrak{R}}_m(\mathcal{F}_{\mathcal{G}, \mu}^{\mathcal{H}}) \leq \mathbbm{E}_{g \sim \mu} \left[ \hat{\mathfrak{R}}_m (\mathcal{H})\right] \leq O\left(\sqrt{\frac{d \ln(\frac{m}{d})}{m}}\right).$ Putting things together, we get $\hat{\mathfrak{R}}_m(\mathcal{L}^{\mathcal{H}}_{\mathcal{G}, \mu}) = \hat{\mathfrak{R}}_m(\ell \circ \mathcal{F}_{\mathcal{G}, \mu}^{\mathcal{H}}) \leq O\left( \sqrt{\frac{dL^2 \ln(\frac{m}{d})}{m}}\right).$
Proposition \ref{thm:rad} then implies that with probability $1-\delta$ over a sample $S \sim \mathcal{D}^m$ of size $m = O\left(\frac{dL^2 \ln(\frac{L}{\epsilon}) + \ln(\frac{1}{\delta})}{\epsilon^2}\right)$, we have 
$\left|\mathbbm{E}_{\mathcal{D}}\left[\ell_{\G, \mu}(h, (x, y)) \right] - \hat{\mathbbm{E}}_{S}\left[\ell_{\G, \mu}(h, (x, y)) \right]\right| \leq \epsilon $ 
for all $h \in \mathcal{H}$ simultaneously.
\end{proof}
We note that \citep{pmlr-v97-yin19b} was the first to use Rademacher complexity to study generalization guarantees in the adversarial robustness setting. Next, we show that uniform convergence of Lipschitz-losses immediately implies proper learning via ERM.
%
%

\begin{theorem}
\label{thm:learnlip}
Let  $\ell_{\mathcal{G}, \mu}(h, (x, y)) = \ell(y\mathbbm{E}_{g \sim \mu}\left[h(g(x)) \right])$ s.t. $\ell: \mathbbm{R} \rightarrow \mathbbm{R}$ 
is a $L$-Lipschitz function.  For every hypothesis class $\mathcal{H}$,  adversary $(\mathcal{G}, \mu)$, and  $(\epsilon, \delta) \in (0, 1)^2$, the proper learning rule $\mathcal{A}(S) = \argmin_{h \in \mathcal{H}}\hat{\mathbbm{E}}_S\left[\ell_{\mathcal{G}, \mu}(h, (x, y)) \right]$, for any distribution $\mathcal{D}$ over $\mathcal{X} \times \mathcal{Y}$, achieves, with probability at least $1 - \delta$ over a sample $S \sim \mathcal{D}^n$ of size $n \geq O\left( \frac{\text{VC}({\mathcal{H}})L^2 \ln(\frac{L}{\epsilon}) + \ln(\frac{1}{\delta})}{\epsilon^2} \right)$, the guarantee  
$$\mathbbm{E}_{\mathcal{D}}\left[\ell_{\G, \mu}(\mathcal{A}(S), (x, y)) \right] \leq \inf_{h \in \mathcal{H}}\mathbbm{E}_{\mathcal{D}}\left[\ell_{\G, \mu}(h, (x, y)) \right]  + \epsilon.$$
%
\end{theorem}

At a high-level, Theorem \ref{thm:learnlip} shows finite VC dimension is sufficient for achieving robustness \textit{between} the average- and worst-case using ERM. In fact, the next theorem, whose proof can be found in Appendix \ref{app:notnec}, shows that finite VC dimension may not even be necessary for this to be true. 
 
\begin{theorem}
\label{thm:notnec}
Let  $\ell_{\mathcal{G}, \mu}(h, (x, y)) = \ell(y\mathbbm{E}_{g \sim \mu}\left[h(g(x))\right])$ s.t. $\ell: \mathbbm{R} \rightarrow \mathbbm{R}$ 
is a $L$-Lipschitz function. There exists $\mathcal{H}$ and $(\mathcal{G}, \mu)$ s.t. $\text{VC}(\mathcal{H}) = \infty$ but $\mathcal{H}$ is still (properly) learnable under $\ell_{\mathcal{G}, \mu}(h, (x, y))$.
\end{theorem}

Together, Theorems \ref{thm:learnlip} and \ref{thm:notnec} showcase an interesting trade-off.  Theorem \ref{thm:notnec} indicates that by carefully choosing $(\mathcal{G}, \mu)$, the complexity of $\mathcal{H}$ can be essentially smoothed out. On the other hand, Theorem \ref{thm:learnlip} shows that any complexity in $(\mathcal{G}, \mu)$ can be smoothed out if $\mathcal{H}$ has finite VC dimension.  This interplay between the complexities of $\mathcal{H}$ and $(\mathcal{G}, \mu)$ closely matches the intuition of \citet{chapelle2000vicinal} in their work on Vicinal Risk Minimization.
Note that the results in this section do not contradict that of Section \ref{sec:relaxedlossnothelp} because $\ell^{\rho}_{\G, \mu}(h, (x, y))$ is a \textit{non-Lipschitz} function of $y\mathbbm{E}_{g \sim \mu}\left[h(g(x))\right]$.

\section{Proper Learnability Under Relaxed Competition} 
\label{sec:relaxedcomp}

The results of Section \ref{sec:relaxedlossnothelp} show that relaxing the worst-case adversarial loss may not always enable proper learning, even for very natural robust loss relaxations. In this section, we show that this bottleneck can be alleviated if we also allow the learner to compete against a slightly stronger notion of robustness. Furthermore, we expand on this idea by exploring other robust learning settings where allowing the learner to compete against a stronger notion of robustness enables proper learnability.
We denote this type of modification to the standard worst-case and probabilistic robustness setting as robust learning under \textit{relaxed competition}. Prior works on Tolerantly Robust PAC Learning \citep{ashtiani2022adversarially, bhattacharjee2022robust} mentioned in the introduction fit under this umbrella. 

Our main tool in this section is Lemma \ref{lem:sand}, which we term as sandwich uniform convergence (SUC). Roughly speaking, SUC provides a sufficient condition under which ERM outputs a predictor that generalizes well w.r.t. a stricter notion of loss. A special case of SUC has implicitly appeared in margin theory (e.g., see \citet[Section 5.4]{mohri2018foundations}), where one evaluates the $0$-$1$ risk of the output hypothesis against the optimal \textit{margin} $0$-$1$ risk.  

\begin{lemma} [Sandwich Uniform Convergence]
\label{lem:sand}
Let $\ell_1(h, (x, y))$ and\
$\ell_2(h, (x, y))$ be bounded, non-negative loss functions s.t. for all $h\in \mathcal{H}$ and $(x, y) \in \mathcal{X} \times \mathcal{Y}$, we have $\ell_1(h, (x, y)) \leq \ell_2(h, (x, y)) \leq 1$. If \emph{there exists} a loss function $\tilde{\ell}(h, (x, y))$ s.t. $\ell_1(h, (x, y)) \leq \tilde{\ell}(h, (x, y)) \leq \ell_2(h, (x, y))$ and $\tilde{\ell}(h, (x, y))$ enjoys the \emph{uniform convergence} property with sample complexity $n(\epsilon, \delta)$, then the learning rule $\mathcal{A}(S) = \inf_{h \in \mathcal{H}} \hat{\mathbbm{E}}_S\left[\ell_2(h, (x, y))\right]$ achieves,  with probability $1-\delta$ over a sample $S \sim \mathcal{D}^m$ of size $m \geq n(\epsilon/2, \delta/2) + O\left(\frac{\ln(\frac{1}{\delta})}{\epsilon^2}\right)$, the guarantee
$$\mathbbm{E}_{\mathcal{D}}\left[\ell_1(\mathcal{A}(S), (x, y))\right] - \inf_{h \in \mathcal{H}}\mathbbm{E}_{\mathcal{D}}\left[\ell_2(h, (x, y))\right] \leq \epsilon.$$ 
%
%
\end{lemma}

Lemma \ref{lem:sand}, whose proof is included in Appendix \ref{app: sanduc}, only requires the \textit{existence} of such a sandwiched loss function that enjoys uniform convergence---we do not actually require it to be computable. In the next two sections, we exploit this fact to give three new generalization guarantees for the empirical risk minimizer over the worst-case robust loss $\ell_{\mathcal{G}}(h, (x ,y))$ and $\rho$-probabilistic robust loss $\ell^{\rho}_{\mathcal{G}, \mu}(h, (x ,y))$, hereafter denoted by $\text{RERM}(S; \mathcal{G}): = \argmin_{h \in \mathcal{H}}\hat{\mathbbm{E}}_S\left[\ell_{\mathcal{G}}(h, (x, y)) \right]$ and $\text{PRERM}(S; \mathcal{G}, \rho): = \argmin_{h \in \mathcal{H}}\hat{\mathbbm{E}}_S\left[\ell^{\rho}_{\mathcal{G, \mu}}(h, (x, y)) \right]$ respectively.

\subsection{$(\rho, \rho^*)$-Probabilistically Robust PAC Learning}\label{sec:relaxed}

In light of the hardness result of Section \ref{sec:relaxedlossnothelp}, we slightly tweak the learning setup in Definition \ref{def:PRPAC} by allowing $\mathcal{A}$ to compete against the hypothesis minimizing the probabilistic robust risk at a level $\rho^* < \rho$. Under this further relaxation, we show that proper learning becomes possible, and that too, via PRERM. In particular, Theorem \ref{thm:properprobrelax} shows that while VC classes are not properly $\rho$-probabilistically robust PAC learnable, they are properly $(\rho, \rho^*)$-probabilistically robust PAC learnable. 

\begin{theorem}[Proper $(\rho, \rho^*)$-Probabilistically Robust PAC Learner]
\label{thm:properprobrelax}
Let $0 \leq \rho^* < \rho$. Then, for every hypothesis class $\mathcal{H}$, adversary $(\mathcal{G}, \mu)$, and $(\epsilon, \delta) \in (0, 1)^2$, the proper learning rule $\mathcal{A}(S) = \text{PRERM}(S; \mathcal{G}, \rho^*)$, for any distribution $\mathcal{D}$ over $\mathcal{X} \times \mathcal{Y}$, achieves, with probability at least $1 - \delta$ over a sample $S \sim \mathcal{D}^n$ of size $n \geq O\left(\frac{\frac{\text{VC}(\mathcal{H})}{(\rho-\rho^*)^2} \ln(\frac{1}{(\rho-\rho^*)\epsilon}) + \ln(\frac{1}{\delta})}{\epsilon^2}\right)$, the guarantee
$$R^{\rho}_{\mathcal{G}, \mu}(\mathcal{A}(S); \mathcal{D}) \leq \inf_{h \in \mathcal{H}}R^{\rho^*}_{\G, \mu}(h;\mathcal{D}) + \epsilon.$$
%
\end{theorem}

In contrast to Section \ref{sec:relaxedlossnothelp}, where proper learning is not always possible, Theorem \ref{thm:properprobrelax}  shows that if we compare our learner to the best hypothesis for a \textit{slightly} stronger level of probabilistic robustness, then not only is proper learning possible for VC classes, but it is possible via an ERM-based learner. Our main technique to prove Theorem \ref{thm:properprobrelax} is
to consider a \textit{different} probabilistically robust loss function that is (1) a Lipschitz function of $y\mathbbm{E}_{g \sim \mu}\left[h(g(x)) \neq y\right]$ and (2) can be sandwiched in between $\ell_{\mathcal{G}, \mu}^{\rho^*}$ and $\ell_{\mathcal{G}, \mu}^{\rho}$. Then, Theorem \ref{thm:properprobrelax} follows from  Lemma \ref{lem:sand}. The full proof is in Appendix \ref{app:properprobrelax}. 
  


\subsection{$(\rho, \G)$-Probabilistically Robust PAC Learning}\label{sec:adv_relaxed}
Can measure-independent learning guarantees be achieved if we instead compare the learner's probabilistically robust risk $R_{\mathcal{G}, \mu}^{\rho}$ to the best  \textit{adversarially robust risk} $R_{\mathcal{G}}$ over $\mathcal{H}$?  We answer this in the affirmative by using SUC. We show that if one wants to compete against the best hypothesis for the worst-case adversarial robust risk, it is sufficient to run RERM.

\begin{theorem} [Proper $(\rho, \G)$-Probabilistically Robust PAC Learner]
\label{thm:advrobrelax}
For every hypothesis class $\mathcal{H}$, adversary $\mathcal{G}$, and $(\epsilon, \delta) \in (0, 1)^2$, the proper learning rule $\mathcal{A}(S) = \emph{\text{RERM}}(S; \mathcal{G})$, for any measure $\mu$ over $\mathcal{G}$ and any distribution $\mathcal{D}$ over $\mathcal{X} \times \mathcal{Y}$, achieves, with probability at least $1 - \delta$ over a sample $S \sim \mathcal{D}^n$ of size $n \geq O\left( \frac{\frac{\text{VC}({\mathcal{H}})}{\rho^2} \ln(\frac{1}{\rho \epsilon}) + \ln(\frac{1}{\delta})}{\epsilon^2} \right)$, the guarantee  
$$R^{\rho}_{\mathcal{G}, \mu}(\mathcal{A}(S); \mathcal{D}) \leq  \inf_{h \in \mathcal{H}}R_{\G}(h;\mathcal{D}) + \epsilon.$$
%
%
\end{theorem} 
The proof of Theorem \ref{thm:advrobrelax} can be found in Appendix \ref{app:advroberelax}, which follows directly from Lemma \ref{lem:sand} by a suitable choice of the sandwiched loss $\ell$. We make a few remarks about the practical importance of Theorem \ref{thm:advrobrelax}. Theorem \ref{thm:advrobrelax} implies that for any pre-specified perturbation function class $\mathcal{G}$ (for example $\ell_p$ balls), running RERM is sufficient to obtain a hypothesis that is probabilistically robust w.r.t. $\textit{any}$ fixed measure $\mu$ over $\mathcal{G}$. Moreover, the level of robustness of the predictor output by RERM,  as measured by $1-\rho$,  scales directly with the sample size - the more samples one has, the smaller $\rho$ can be made. Alternatively, for a fixed sample size $m$, desired error $\epsilon$ and confidence $\delta$, one can use the sample complexity guarantee in Theorem \ref{thm:advrobrelax} to back-solve the robustness guarantee $\rho$.

\subsection{Tolerantly Robust PAC Learning}
 In Tolerantly Robust PAC Learning  \citep{bhattacharjee2022robust, ashtiani2022adversarially}, the learner's adversarially robust risk under a perturbation set $\mathcal{G}$ is compared with the best achievable adversarial robust risk for a larger perturbation set $\mathcal{G}' \supset \mathcal{G}$. \cite{ashtiani2022adversarially} study the setting where both $\mathcal{G}$ and $\mathcal{G}'$ induce $\ell_p$ balls with radius $r$ and $(1+\gamma)r$ respectively.  In the work of \cite{bhattacharjee2022robust}, $\mathcal{G}$ is arbitrary, but $\mathcal{G}'$ is constructed such that it induces perturbation sets that are the union of balls with radius $\gamma$ that cover $\mathcal{G}$. Critically, \citet{bhattacharjee2022robust} show that, under certain assumptions, running RERM over a larger perturbation set $\mathcal{G}'$ is sufficient for Tolerantly Robust PAC learning.  In this section, we take a slightly different approach to Tolerantly Robust PAC learning. Instead of having the learner compete against the best possible risk for a larger perturbation set, we have the learner compete against the best possible adversarial robust risk for $\G$, but evaluate the learner's adversarial robust risk using a \textit{smaller} perturbation set $\G' \subset \G$. 

For what $\mathcal{G}' \subset \mathcal{G}$ is Tolerantly Robust PAC learning via RERM possible? As an immediate result of Lemma \ref{lem:sand} and Vapnik's ``General Learning'', finite VC dimension of the loss class $\mathcal{L}_{\mathcal{G}'}^{\mathcal{H}} = \{(x, y) \mapsto \ell_{\G}(h, (x,y)): h \in \mathcal{H}\}$ is sufficient. Note that finite VC dimension of $\mathcal{L}_{\mathcal{G}'}^{\mathcal{H}}$ implies that the loss function $\ell_{\mathcal{G}'}(h, (x, y))$ enjoys the uniform convergence property with sample complexity $O\left(\frac{\text{VC}(\mathcal{L}^{\mathcal{G}'}_{\mathcal{H}}) + \ln(\frac{1}{\delta})}{\epsilon^2}\right)$. Thus, taking $\ell_1(h, (x, y)) = \tilde{\ell}(h, (x, y)) = \ell_{\mathcal{G}'}(h, (x, y))$ and $\ell_2(h, (x, y)) = \ell_{\mathcal{G}}(h, (x, y))$ in Lemma \ref{lem:sand}, we have that if there exists a $\mathcal{G}' \subset \mathcal{G}$ s.t. $\text{VC}(\mathcal{L}^{\mathcal{G}'}_{\mathcal{H}}) < \infty$, then with probability $1 - \delta$ over a sample $S \sim \mathcal{D}^n$ of size $n = O\left(\frac{\text{VC}(\mathcal{L}_{\mathcal{G}'}^{\mathcal{H}}) + \ln(\frac{1}{\delta})}{\epsilon^2}\right)$, $R_{\mathcal{G'}}(\mathcal{A}(S); \mathcal{D}) \leq  \inf_{h \in \mathcal{H}}R_{\G}(h;\mathcal{D}) + \epsilon$, where $\mathcal{A}(S) = \text{RERM}(S; \mathcal{G})$. 

Alternatively, if $\mathcal{G}' \subset \mathcal{G}$ such that there exists a \textit{finite} subset $\tilde{\mathcal{G}} \subset \mathcal{G}$ where $\ell_{\mathcal{G}'}(h, (x, y)) \leq \ell_{\tilde{\mathcal{G}}}(h, (x, y))$, then Tolerantly Robust PAC learning via RERM is possible with sample complexity that scales according to $O\left(\frac{\text{VC}(\mathcal{H})\log(|\tilde{\mathcal{G}}|) + \ln(\frac{1}{\delta})}{\epsilon^2}\right)$. This result essentially comes from the fact that the VC dimension of the loss class for any finite perturbation set $\tilde{\mathcal{G}}$ incurs only a $\log(|\tilde{\mathcal{G}}|)$ blow-up  from the VC dimension of $\mathcal{H}$ (see Lemma 1.1 in \cite{attias2021improved}). Thus, finite VC dimension of $\mathcal{H}$ implies finite VC dimension of the loss class $\mathcal{L}_{\mathcal{H}}^{\tilde{\mathcal{G}}}$ which implies uniform convergence of the loss $\ell_{\tilde{\mathcal{G}}}(h, (x, y))$, as needed for Lemma \ref{lem:sand} to hold. 

We now give an example where such a finite approximation of $\mathcal{G}'$ is possible. In order to do so, we will need to consider a \textit{metric space} of perturbation functions $(\mathcal{G}, d)$ and  define a notion of ``nice'' perturbation sets, similar to``regular'' hypothesis classes from \cite{bhattacharjee2022robust}.

\begin{definition} [$r$-Nice Perturbation Set]
\label{def: r-nice}
Let $\mathcal{H}$ be a hypothesis class and $(\mathcal{G}, d)$ a metric space of perturbation functions. Let $B_{r}(g) := \{g' \in \mathcal{G}:  d(g, g') \leq r\}$ denote a closed ball of radius $r$ centered around $g \in \mathcal{G}$. We say that $\mathcal{G}' \subset \mathcal{G}$ is $r$-\emph{Nice} w.r.t. $\mathcal{H}$, if for all $x \in \mathcal{X}$, $h \in \mathcal{H}$, and $g \in \mathcal{G}'$, there exists a $g^* \in \mathcal{G}$, such that $g \in B_r(g^*)$ and  $h(g(x)) = h(g'(x))$ for all $g' \in B_r(g^*)$. 
\end{definition}

Definition \ref{def: r-nice} prevents a situation where a hypothesis $h \in \mathcal{H}$ is non-robust to an isolated perturbation  function $g \in \mathcal{G}'$ for any given labelled example $(x, y) \in \mathcal{X} \times \mathcal{Y}$. If a hypothesis $h$ is non-robust to a perturbation $g \in \mathcal{G'}$, then Definition \ref{def: r-nice} asserts that there must exist a small ball of perturbation functions in $\mathcal{G}$ over which $h$ is also non-robust. Next, we define the covering number.

\begin{definition} [Covering Number]
Let $(\mathcal{M}, d)$ be a metric space, let $\mathcal{K} \subset \mathcal{M}$ be a subset, and $r > 0$. Let $B_r(x) = \{x' \in \mathcal{M}:  d(x, x') \leq r\}$ denote the ball of radius $r$ centered around $x \in \mathcal{M}$. A subset $\mathcal{C} \subset \mathcal{M}$ is an $r$-covering of $\mathcal{K}$ if $\mathcal{K} \subset \bigcup_{c \in \mathcal{C}}B_r(c)$. The \emph{covering number} of $\mathcal{K}$, denoted $\mathcal{N}_r(\mathcal{K}, d)$, is the smallest cardinality of any $r$-covering of $\mathcal{K}$. 
\end{definition}

Finally, let $\mathcal{G}'_{2r} = \bigcup_{g \in \mathcal{G}'}B_{2r}(g)$ denote the union over all balls of radius $2r$ with centers in $\mathcal{G}'$. Theorem \ref{thm:tolerant} then states that if there exists a set $\mathcal{G}' \subset \mathcal{G}$ that is $r$-Nice w.r.t. $\mathcal{H}$, then Tolerantly Robust PAC learning is possible via RERM with sample complexity that scales logarithmically with $\mathcal{N}_{r}(\mathcal{G}'_{2r}, d)$. In Appendix \ref{app:relaxedcomp}, we give a full proof and show that \emph{$\ell_p$ balls are $r$-Nice perturbation sets for robustly learning halfspaces}. 

\begin{theorem} [Tolerantly Robust PAC learning under Nice Perturbations]
\label{thm:tolerant}
Let $\mathcal{H} \subset \mathcal{Y}^{\mathcal{X}}$ be a hypothesis class and $(\mathcal{G}, d)$ be a metric space of perturbation functions. If there exists a subset $\mathcal{G}' \subset \mathcal{G}$ such that $\mathcal{G}'$ is $r$-Nice w.r.t. $\mathcal{H}$, then the proper learning rule $\mathcal{A}(S) = \emph{\text{RERM}}(S; \mathcal{G})$, for any  distribution $\mathcal{D}$ over $\mathcal{X} \times \mathcal{Y}$, achieves, with probability at least $1 - \delta$ over a sample $S \sim \mathcal{D}^n$ of size $n \geq O\left(\frac{\text{VC}(\mathcal{H})\log(\mathcal{N}_{r}(\mathcal{G}'_{2r}, d)) + \ln(\frac{1}{\delta})}{\epsilon^2}\right)$, the guarantee  
$$R_{\mathcal{G}'}(\mathcal{A}(S); \mathcal{D}) \leq  \inf_{h \in \mathcal{H}}R_{\G}(h;\mathcal{D}) + \epsilon.$$
%
\end{theorem}

\section{Conclusion}
In this work,  we show that there exists natural robust loss relaxations for which finite VC dimension is still not sufficient for proper learning. On the other hand, we identify a large set of Lipschitz robust loss relaxations for which finite VC dimension is sufficient for proper learnability. In addition, we give new generalization guarantees for the adversarially robust empirical risk minimizer.  As future work, we are interested in understanding whether our robust loss relaxations can be used to mitigate the tradeoff between achieving adversarial robustness and maintaining high nominal performance. 

\bibliographystyle{plainnat}
\bibliography{sample}

\newpage
\appendix

\section{Equivalence between Adversarial Robustness Models}
\label{app:eq}
We show that the perturbation set and perturbation function models are equivalent. 

\begin{theorem} [Equivalence between $\mathcal{G}$ and $\mathcal{U}$]
Let $\mathcal{X}$ be an arbitrary domain. There exists a perturbation set $\U: \mathcal{X} \rightarrow 2^{\mathcal{X}}$ if and only if there exists a set of perturbation functions $\mathcal{G}$ s.t. $\mathcal{G}(x) = \{g(x): g \in \mathcal{G}\} = \U(x)$ for all $x \in \mathcal{X}$.
\end{theorem}

\begin{proof} We first show that every set of perturbation functions $\G$ induces a perturbation set $\mathcal{U}$. Let $\G$ be an arbitrary set of perturbation functions $g: \mathcal{X} \to \mathcal{X}$. Then, for each $x \in \mathcal{X}$,  define $\U(x) := \{g(x): g \in \mathcal{G}\}$, which completes the proof of this direction. 

Now we will show the converse - every perturbation set $\U$ induces a point-wise equivalent set $\mathcal{G}$ of perturbation functions. Let $\U$ be an arbitrary perturbation set mapping points in $\mathcal{X}$ to subsets in $\mathcal{X}$. Assume that $\U(x)$ is not empty for all $x \in \mathcal{X}$. Let $\tilde{z}_x$ denote an arbitrary perturbation from $\U(x)$.  For every $x \in \mathcal{X}$, and every $z \in \U(x)$, define the perturbation function $g^x_z(t) = z\mathbbm{1}\{t = x\} + \tilde{z}_t\mathbbm{1}\{t \neq x\}$ for $t \in \mathcal{X}$. Observe that $g_z^x(x) = z \in \U(x)$ and $g_z^x(x') = \tilde{z}_{x'} \in \U(x')$. Finally, let $\mathcal{G} = \bigcup_{x \in \mathcal{X}} \bigcup_{z \in \U(x)} \{g_z^x\}$. To verify that $\mathcal{G} = \U$, consider an arbitrary point $x' \in \mathcal{X}$. Then, 
\begin{align*}
\mathcal{G}(x') &= \bigcup_{x \in \mathcal{X}} \bigcup_{z \in \U(x)} \{g_z^x(x')\}\\
&= \left(\bigcup_{z \in \U(x')} \{g_z^{x'}(x')\}\right) \cup \left(\bigcup_{x \in \mathcal{X} \setminus x'} \bigcup_{z \in \U(x)} \{g_z^x(x')\}\right)\\
&= \left(\bigcup_{z \in \U(x')} \{z\}\right) \cup \left(\bigcup_{x \in \mathcal{X} \setminus x'} \bigcup_{z \in \U(x)} \{\tilde{z}_{x'}\}\right)\\
&= \U(x') \cup \tilde{z}_{x'}\\
&= \U(x').
\end{align*}
as needed. 
\end{proof}

\section{Proofs for Section \ref{sec:relaxedloss}}
\subsection{Proof of Theorem \ref{thm:learnlip}}
\begin{proof} (of Theorem \ref{thm:learnlip})
Let $\text{VC}(\mathcal{H}) = d$ and $S = \{(x_1, y_1), ..., (x_m, y_m)\}$ an i.i.d. sample of size $m$ from $\mathcal{D}$. Consider the learning algorithm $\mathcal{A}(S) = \argmin_{h \in \mathcal{H}}\hat{\mathbbm{E}}_S\left[\ell_{\mathcal{G}, \mu}(h, (x, y)) \right]$. Note that $\mathcal{A}$ is a proper learning algorithm. Let $\hat{h} = \mathcal{A}(S)$ denote hypothesis output by $\mathcal{A}$ and $h^* = \inf_{h \in \mathcal{H}}\mathbbm{E}_{\mathcal{D}}\left[\ell_{\G, \mu}(h, (x, y)) \right]$. 

We now show that if the sample size $m = O\left(\frac{dL^2 \ln(\frac{L}{\epsilon}) + \ln(\frac{1}{\delta})}{\epsilon^2}\right)$, then $\hat{h}$ achieves the stated generalization bound with probability $1 - \delta$. By Lemma \ref{lem: uclip}, if $m = O\left(\frac{dL^2 \ln(\frac{L}{\epsilon}) + \ln(\frac{1}{\delta})}{\epsilon^2}\right)$, we have that with probability $1 - \delta$, for all $h \in \mathcal{H}$ simultaneously, 
$$\left|\mathbbm{E}_{\mathcal{D}}\left[\ell_{\G, \mu}(h, (x, y)) \right] - \hat{\mathbbm{E}}_{S}\left[\ell_{\G, \mu}(h, (x, y)) \right]\right| \leq \frac{\epsilon}{2}.$$
This means that both $\mathbbm{E}_{\mathcal{D}}\left[\ell_{\G, \mu}(\hat{h}, (x, y)) \right] - \hat{\mathbbm{E}}_{S}\left[\ell_{\G, \mu}(\hat{h}, (x, y)) \right] \leq \frac{\epsilon}{2}$ and $ \hat{\mathbbm{E}}_{S}\left[\ell_{\G, \mu}(h^*, (x, y)) \right] - \mathbbm{E}_{\mathcal{D}}\left[\ell_{\G, \mu}(h^*, (x, y)) \right] \leq \frac{\epsilon}{2}$. By definition of $\hat{h}$, note that $\hat{\mathbbm{E}}_{S}\left[\ell_{\G, \mu}(\hat{h}, (x, y)) \right] \leq \hat{\mathbbm{E}}_{S}\left[\ell_{\G, \mu}(h^*, (x, y)) \right]$. Putting these observations together, we have that 
\begin{align*}
\mathbbm{E}_{\mathcal{D}}\left[\ell_{\G, \mu}(\hat{h}, (x, y)) \right] - (\mathbbm{E}_{\mathcal{D}}\left[\ell_{\G, \mu}(h^*, (x, y)) \right]+ \frac{\epsilon}{2}) &\leq \mathbbm{E}_{\mathcal{D}}\left[\ell_{\G, \mu}(\hat{h}, (x, y)) \right] - \hat{\mathbbm{E}}_{S}\left[\ell_{\G, \mu}(h^*, (x, y)) \right]\\
&\leq \mathbbm{E}_{\mathcal{D}}\left[\ell_{\G, \mu}(\hat{h}, (x, y)) \right] - \hat{\mathbbm{E}}_{S}\left[\ell_{\G, \mu}(\hat{h}, (x, y)) \right]\\
&\leq \frac{\epsilon}{2},
\end{align*}
from which we can deduce that 
$$\mathbbm{E}_{\mathcal{D}}\left[\ell_{\G, \mu}(\hat{h}, (x, y)) \right] - \inf_{h \in \mathcal{H}}\mathbbm{E}_{\mathcal{D}}\left[\ell_{\G, \mu}(h, (x, y)) \right]  \leq \epsilon.$$

Thus, $\mathcal{A}$ achieves the stated generalization bound with sample complexity $m = O\left(\frac{dL^2 \ln(\frac{L}{\epsilon}) + \ln(\frac{1}{\delta})}{\epsilon^2}\right)$, completing the proof.
\end{proof}

\subsection{Proof of Theorem \ref{thm:notnec}}
\label{app:notnec}
For the proof in this section, it will be useful to define the $(\mathcal{G}, \mu)$-smoothed hypothesis class $\mathcal{H}$:
$$\mathcal{F}^{\mathcal{H}}_{\mathcal{G}, \mu} := \{\mathbbm{E}_{g \sim \mu}\left[h(g(x))\right]: h \in \mathcal{H}\}.$$
\begin{proof} (of Theorem \ref{thm:notnec}) Let $\mathcal{X} = \mathbbm{R}$ and $\mathcal{H} = \{\text{sign}(\sin(\omega x)): \omega \in \mathbbm{R}\}$. Without loss of generality, assume $\text{sign}(\sin(0)) = 1$. For every $x \in \mathcal{X}$ and $c \in [-1, 1]$, define $g_c(x) = cx$. Then, let $\mathcal{G} = \{g_c: c \in [-1, 1]\}$ and $\mu$ be uniform over $\mathcal{G}$. First, $\text{VC}(\mathcal{H}) = \infty$ as desired. Next, to show learnability, it suffices to show that the loss 

$$\ell_{\G, \mu}(h, (x, y)) = \ell(y\mathbbm{E}_{g \sim \mu}\left[h(g(x))\right]). $$
enjoys the uniform convergence property despite $\text{VC}(\mathcal{H}) = \infty$. By Theorem \ref{thm:rad} and similar to the proof of Lemma \ref{lem: uclip}, it suffices upperbound the Rademacher complexity of the loss class $\mathcal{L}^{\mathcal{H}}_{\mathcal{G}, \mu} = \{(x, y) \mapsto \ell_{\G, \mu}(h, (x, y)): h \in \mathcal{H}\}$. Since for every fixed $y$, $\ell_{\mathcal{G}, \mu}(h, (x, y))$ is $L$-Lipschitz w.r.t the real-valued function $\mathbbm{E}_{g \sim \mu}\left[h(g(x)) \right]$, by Ledoux-Talagrand's contraction principle $\hat{\mathfrak{R}}_m(\mathcal{L}_{\mathcal{G}, \mu}^{\mathcal{H}}) \leq L \cdot \hat{\mathfrak{R}}_m(\mathcal{F}_{\mathcal{G}, \mu}^{\mathcal{H}})$ where $\mathcal{F}_{\mathcal{G}, \mu}^{\mathcal{H}}$ is the $(\mathcal{G}, \mu)$-smoothed hypothesis classed defined previously. Thus, it suffices to upper-bound $\hat{\mathfrak{R}}_m(\mathcal{F}_{\mathcal{G}, \mu}^{\mathcal{H}})$ by a sublinear function of $m$ to show that $\ell_{\mathcal{G}, \mu}(h, (x, y))$ enjoys the uniform convergence property. But for every $h_{\omega} \in \mathcal{H}$, 

$$\mathbbm{E}_{g \sim \mu}\left[h_{\omega}(g(x)) \right] = \mathbbm{E}_{c \sim \text{Unif}(-1, 1)}\left[\text{sign}(\sin(\omega(cx))\right] = \frac{1}{2}\int_{-1}^{1}\text{sign}(\sin(c(\omega x))) dc.$$

Since $\sin(ax)$ is an odd function, $\text{sign}(\sin(ax))$ is also odd, from which it follows that for all $h_{\omega} \in \mathcal{H}$:
$$\mathbbm{E}_{g \sim \mu}\left[h_{\omega}(g(x)) \right] = \begin{cases}
0 & \text{if $x \neq 0 \text{ and }\omega \neq 0$}\\
1 & \text{otherwise}
\end{cases}.$$

Therefore, $\mathcal{F}_{\mathcal{G}, \mu}^{\mathcal{H} } = \{f_1, f_2\}$ where $f_1(x) = 1$ for all $x\in \mathbbm{R}$ and $f_2(x) = 1$ if $x = 0$ and $f_2(x) = 0$ if $x \neq 0$. Since $\mathcal{F}_{\mathcal{G}, \mu}^{\mathcal{H}}$ is finite, by Massart's Lemma \citep{mohri2018foundations}, $\hat{\mathfrak{R}}_m(\mathcal{F}_{\mathcal{G}, \mu}^{\mathcal{H}})$ is upper-bounded by a sublinear function of $m$ such that $\ell_{\mathcal{G}, \mu}(h, (x, y))$ enjoys the uniform convergence property with sample complexity $O(\frac{L^2 + \ln(\frac{1}{\delta})}{\epsilon^2})$. Therefore, $(\mathcal{H}, \mathcal{G}, \mu)$ is PAC learnable w.r.t $\ell_{\G, \mu}(h, (x, y))$ by the learning rule $\mathcal{A}(S) = \argmin_{h \in \mathcal{H}}\hat{\mathbbm{E}}_S\left[\ell_{\G, \mu}(h, (x, y)) \right]$ with sample complexity that scales according to $O(\frac{L^2 + \ln(\frac{1}{\delta})}{\epsilon^2})$.
\end{proof}

\section{Proofs for Section \ref{sec:relaxedlossnothelp}}
\subsection{Proper $\rho$-Probabilistically Robust PAC Learning}
\label{app:finiteG}

We show that if  $\mathcal{G}$ is \textit{finite} then VC classes are $\rho$-probabilistically robustly learnable.

\begin{theorem}[Proper $\rho$-Probabilistically Robust PAC Learner]
\label{thm:finiteG}
For every hypothesis class $\mathcal{H}$, threshold $\rho \in [0, 1)$, and adversary $(\mathcal{G}, \mu)$ s.t. $|\mathcal{G}| \leq K$, there exists a proper learning rule $\mathcal{A}: (\mathcal{X} \times \mathcal{Y})^n \rightarrow \mathcal{H}$ such that for every distribution $\mathcal{D}$ over $\mathcal{X} \times \mathcal{Y}$, with probability at least $1 - \delta$ over $S \sim \mathcal{D}^n$, algorithm $\mathcal{A}$ achieves
$$R^{\rho}_{\mathcal{G}, \mu}(\mathcal{A}(S); \mathcal{D}) \leq \inf_{h \in \mathcal{H}}R^{\rho}_{\G, \mu}(h;\mathcal{D}) + \epsilon$$
with 
$$n(\epsilon, \delta, \rho; \mathcal{H}, \mathcal{G}, \mu) = O\left(\frac{\text{VC}(\mathcal{H})\ln(K) + \ln(\frac{1}{\delta})}{\epsilon^2}\right)$$
samples.
\end{theorem}
\begin{proof}
Fix $\rho \in (0, 1)$. Our main strategy will be to upper bound the VC dimension of the $\rho$-probabilistically robust loss class by some function of the VC dimension of $\mathcal{H}$. Then, finite VC dimension of $\mathcal{H}$ implies finite VC dimension of the loss class, which ultimately implies uniform convergence over the $\rho$-probabilistically robust loss. Finally,  uniform convergence of $\ell^{\rho}_{\G, \mu}(h, (x, y))$ implies that ERM is sufficient for $\rho$-probabilistically robust PAC learning. To that end, define

$$\mathcal{L}_{\mathcal{G}, \mu}^{\mathcal{H}, \rho} = \{(x, y) \mapsto \one \{\mathbbm{P}_{g \sim \mu}\left( h(g(x)) \neq y \right) > \rho \}: h \in \mathcal{H}\}$$

as the $\rho$-probabilistically robust loss class of $\mathcal{H}$. Let $S = \{(x_1, y_1), ...., (x_n, y_n)\} \in (\mathcal{X} \times \mathcal{Y})^n$ be an arbitrary labeled sample of size $n$. Inflate $S$ to $S_{\mathcal{G}}$ by adding for each labelled example $(x, y) \in S$ all possible perturbed examples $(g(x), y)$ for $g \in \mathcal{G}$. That is, $S_{\mathcal{G}} = \bigcup_{(x, y) \in S} \{(g(x), y): g \in \mathcal{G}\}.$ Note that $|S_{\mathcal{G}}| \leq nK.$ Let $\mathcal{L}_{\mathcal{G}, \mu}^{\mathcal{H}, \rho}(S)$ denote the set of all possible behaviors of functions in $\mathcal{L}_{\mathcal{G}, \mu}^{\mathcal{H}, \rho}$ on $S$. Likewise, let $\mathcal{H}(S_{\mathcal{G}})$ denote the set of all possible behaviors of functions in $\mathcal{H}$ on the inflated set $S_{\mathcal{G}}$. Note that each behavior in $\mathcal{L}_{\mathcal{G}, \mu}^{\mathcal{H}, \rho}(S)$ maps to at least $1$ behavior in $\mathcal{H}$. Therefore $|\mathcal{L}_{\mathcal{G}, \mu}^{\mathcal{H}, \rho}(S)| \leq |\mathcal{H}(S_{\mathcal{G}})|$. By Sauer-Shelah's lemma, $|\mathcal{H}(S_{\mathcal{G}})| \leq (nK)^{\text{VC}(\mathcal{H})}$. Solving for $n$ s.t. $(nK)^{\text{VC}(\mathcal{H})} < 2^n$ gives that $n = O(\text{VC}(\mathcal{H})\ln(K))$, ultimately implying that $\text{VC}(\mathcal{L}_{\mathcal{G}, \mu}^{\mathcal{H}, \rho}) \leq O(\text{VC}(\mathcal{H})\ln(K))$ (see Lemma 1.1 in \cite{attias2021improved}).

Since for VC classes, the VC dimension of $\mathcal{L}_{\mathcal{G}, \mu}^{\mathcal{H}, \rho}$ is bounded, by Vapnik's ``General Learning", we have that for VC classes the loss function $\ell^{\rho}_{\G, \mu}(h, (x, y))$ enjoys the uniform convergence property. Namely, let $\mathcal{D}$ be a distribution over $\mathcal{X} \times \mathcal{Y}$. For a sample of size $n \geq O(\frac{\text{VC}(\mathcal{H})\ln(K) + \ln(\frac{1}{\delta})}{\epsilon^2})$, we have that with probability at least $1 - \delta$ over $S \sim \mathcal{D}^n$, for all $h \in \mathcal{H}$

$$|\mathbbm{E}_{\mathcal{D}}\left[\ell^{\rho}_{\G, \mu}(h, (x, y)) \right] - \hat{\mathbbm{E}}_{\mathcal{S}}\left[\ell^{\rho}_{\G, \mu}(h, (x, y)) \right]| \leq \epsilon.$$

Standard arguments yield that the proper learning rule $\mathcal{A}(S) = \argmin_{h \in \mathcal{H}} \hat{\mathbbm{E}}_S\left[\ell^{\rho}_{\G, \mu}(h, (x, y)) \right]$ is a $\rho$-probabilistically robust PAC learner with sample complexity $O(\frac{\text{VC}(\mathcal{H})\ln(K) + \ln(\frac{1}{\delta})}{\epsilon^2}).$
\end{proof}

\subsection{Proof of Lemma \ref{lem2}}
\label{app:rhosecond}
\begin{proof}(of Lemma \ref{lem2}) This proof closely follows Lemma 3 from \cite{montasser2019vc}. In fact, the only difference is in the construction of the hypothesis class, which we will describe below.

Fix $\rho \in [0, 1)$. Let $m \in \mathbbm{N}$. Construct a hypothesis class $\mathcal{H}_0$ as in Lemma \ref{lem1} on $3m$ centers $c_1, ..., c_{3m}$ based on $\rho$. By the construction in Lemma \ref{lem1}, we know that $\mathcal{L}_{\mathcal{G}, \mu}^{\mathcal{H}, \rho}$ shatters the sample $C = \{(c_1, 1), ..., (c_{3m}, 1)\}$. Instead of keeping all of $\mathcal{H}_0$, we will only keep a subset $\mathcal{H}$ of $\mathcal{H}_0$, namely those classifiers that are probabilistically robustly correct on subsets of size $2m$ of $C$. More specifically, recall from the construction in Lemma \ref{lem1}, that each hypothesis $h_b \in \mathcal{H}_0$ is parameterized by a bitstring $b \in \{0, 1\}^{3m}$ where if $b_i = 1$, then $h_b$ is not robust to example $(c_i, 1)$. Therefore, $\mathcal{H} = \{h_b \in \mathcal{H}_0: \sum_{i=1}^{3m}b_i = m\}$. Now, let $\mathcal{A}: (\mathcal{X} \times \mathcal{Y})^* \rightarrow \mathcal{H}$ be an arbitrary proper learning rule. Consider a set of distributions $\mathcal{D}_1, ..., \mathcal{D}_L$ where $L = \binom{3m}{2m}$. Each distribution $\mathcal{D}_i$ is uniform over exactly $2m$ centers in $C$. Critically, note that by our construction of $\mathcal{H}$, every distribution $\mathcal{D}_i$ is probabilistically robustly realizable by a hypothesis in $\mathcal{H}$. That is, for all $\mathcal{D}_i$, there exists a hypothesis $h^* \in \mathcal{H}$ s.t. $R^{\rho}_{\G, \mu}(h^*;\mathcal{D}_i) = 0$. Observe that this satisfies the first condition in Lemma \ref{lem2}. For the second condition, at a high-level, the idea is to use the probabilistic method to show that there exists a distribution $\mathcal{D}_i$ where $\E_{S \sim \mathcal{D}_i^m}\left[ R^{\rho}_{\G, \mu}(\mathcal{A}(S);\mathcal{D})\right] \geq \frac{1}{4}$ and then use a variant of Markov's inequality to show that with probability at least $1/7$ over $S \sim \mathcal{D}^{m}$, $R^{\rho}_{\G, \mu}(\mathcal{A}(S);\mathcal{D}) > 1/8$. 

Let $S \in C^m$ be an arbitrary set of $m$ points. Let $\mathcal{C}$ be a uniform distribution over $C$. Let $\mathcal{P}$ be a uniform distribution over $\mathcal{D}_1, ..., \mathcal{D}_T$. Let $E_S$ denote the event that $S \subset \text{supp}(\mathcal{D}_i)$ for $\mathcal{D}_i \sim \mathcal{P}$. Given the event $E_S$, we will lower bound the expected probabilistic robust loss of the hypothesis the proper learning rule $\mathcal{A}$ outputs, 

$$\mathbbm{E}_{\mathcal{D}_i \sim \mathcal{P}}\left[ R^{\rho}_{\mathcal{G}, \mu}(\mathcal{A}(S); \mathcal{D}_i)| E_S\right] = \mathbbm{E}_{\mathcal{D}_i \sim \mathcal{P}}\left[\mathbbm{E}_{(x, y) \sim \mathcal{D}_i}\left[ \one \{\mathbbm{P}_{g \sim \mu}\left( \mathcal{A}(S)(g(x)) \neq y \right) > \rho \}\right] | E_S\right].$$

Conditioning on the event that $(x, y) \notin S$, denoted, $E_{(x, y) \notin S}$, 

\begin{align*}
\mathbbm{E}_{(x, y) \sim \mathcal{D}_i}\left[ \one \{\mathbbm{P}_{g \sim \mu}\left( \mathcal{A}(S)(g(x)) \neq y \right) > \rho \}\right] &\geq
\mathbbm{P}_{(x, y) \sim \mathcal{D}_i}\left[E_{(x, y) \notin S} \right] \\
&\times \mathbbm{E}_{(x, y) \sim \mathcal{D}_i}\left[\one \{\mathbbm{P}_{g \sim \mu}\left( \mathcal{A}(S)(g(x)) \neq y \right)  >\rho\}| E_{(x, y) \notin S}\right]
\end{align*}

Since $\mathcal{D}_i$ is supported over $2m$ points and $|S| = m$, $\mathbbm{P}_{(x, y) \sim \mathcal{D}_i}\left[E_{(x, y) \notin S} \right] \geq \frac{1}{2}$ since in the worst-case $S \subset \text{supp}(\mathcal{D}_i)$. Thus, we obtain the lower bound, 

$$\mathbbm{E}_{\mathcal{D}_i \sim \mathcal{P}}\left[ R^{\rho}_{\mathcal{G}, \mu}(\mathcal{A}(S); \mathcal{D}_i)| E_S\right] \geq \frac{1}{2}\mathbbm{E}_{\mathcal{D}_i \sim \mathcal{P}}\left[ \mathbbm{E}_{(x, y) \sim \mathcal{D}_i}\left[\one \{\mathbbm{P}_{g \sim \mu}\left( \mathcal{A}(S)(g(x)) \neq y \right)  >\rho\}| E_{(x, y) \notin S}\right]| E_S\right].$$

Unravelling the expectation over the draw from $\mathcal{D}_i$ given the event $E_S$, we have, 

$$\mathbbm{E}_{(x, y) \sim \mathcal{D}_i}\left[\one \{\mathbbm{P}_{g \sim \mu}\left( \mathcal{A}(S)(g(x)) \neq y \right)  >\rho\}| E_{(x, y) \notin S}\right] \geq \frac{1}{m}\sum_{(x, y) \in \text{supp}(\mathcal{D}_i) \setminus S} \one \{\mathbbm{P}_{g \sim \mu}\left( \mathcal{A}(S)(g(x)) \neq y \right)  >\rho\}$$

Observing that $\mathbbm{E}_{\mathcal{D}_i \sim \mathcal{P}}\left[\one \{(x, y) \in \text{supp}(\mathcal{D}_i)\} | E_S\right] \geq \frac{1}{2}$ yields, 

$$\mathbbm{E}_{\mathcal{D}_i \sim \mathcal{P}}\left[ \mathbbm{E}_{(x, y) \sim \mathcal{D}_i}\left[\one \{\mathbbm{P}_{g \sim \mu}\left( \mathcal{A}(S)(g(x)) \neq y \right)  >\rho\}| E_{(x, y) \notin S}\right]| E_S\right] \geq \frac{1}{2m}\sum_{(x, y) \notin S} \one \{\mathbbm{P}_{g \sim \mu}\left( \mathcal{A}(S)(g(x)) \neq y \right)  >\rho\}.$$

Since $\mathcal{A}(S) \in \mathcal{H}$, by construction of $\mathcal{H}$, there are at least $m$ points in $C$ where $\mathcal{A}$ is not probabilistically robustly correct. Therefore, 

$$\frac{1}{2m}\sum_{(x, y) \notin S} \one \{\mathbbm{P}_{g \sim \mu}\left( \mathcal{A}(S)(g(x)) \neq y \right)  >\rho\} \geq \frac{1}{2},$$

from which we have that, $\mathbbm{E}_{\mathcal{D}_i \sim \mathcal{P}}\left[ R^{\rho}_{\mathcal{G}, \mu}(\mathcal{A}(S); \mathcal{D}_i)| E_S\right] \geq \frac{1}{4}$. By the law of total expectation, we have that

\begin{align*}
\mathbbm{E}_{\mathcal{D}_i \sim \mathcal{P}}\left[\mathbbm{E}_{S \sim \mathcal{D}_i^m}\left[R^{\rho}_{\mathcal{G}, \mu}(\mathcal{A}(S); \mathcal{D}_i) \right] \right] &= \mathbbm{E}_{S \sim \mathcal{C}} \left[\mathbbm{E}_{\mathcal{D}_i \sim \mathcal{P}|E_S} \left[R^{\rho}_{\mathcal{G}, \mu}(\mathcal{A}(S); \mathcal{D}_i) \right] \right] \\
&= \mathbbm{E}_{S \sim \mathcal{C}} \left[\mathbbm{E}_{\mathcal{D}_i \sim \mathcal{P}} \left[R^{\rho}_{\mathcal{G}, \mu}(\mathcal{A}(S); \mathcal{D}_i)| E_S \right] \right]\\
&\geq 1/4
\end{align*}

Since the expectation over $\mathcal{D}_1, ..., \mathcal{D}_T$ is at least $1/4$, there must exist a distribution $\mathcal{D}_i$ where $\mathbbm{E}_{S \sim \mathcal{D}_i^m}\left[R^{\rho}_{\mathcal{G}, \mu}(\mathcal{A}(S); \mathcal{D}_i) \right] \geq 1/4$. Using a variant of Markov's inequality, gives $$\mathbbm{P}_{S \sim \mathcal{D}_i^m}\left[R^{\rho}_{\mathcal{G}, \mu}(\mathcal{A}(S); \mathcal{D}_i) > 1/8\right] \geq 1/7$$ which completes the proof. 
\end{proof}

\subsection{Proof of Theorem \ref{thm:propernotposs}}
\label{app:propernotposs}
\begin{proof} (of Theorem \ref{thm:propernotposs})
Fix $\rho \in [0, 1)$. Let $(C_m)_{m \in \mathbbm{N}}$ be an infinite sequence of disjoint sets such that each set $C_m$ contains $3m$ distinct center points from $\mathcal{X}$, where for any $c_i, c_j \in \bigcup_{m=1}^{\infty}C_m$ such that $c_i \neq c_j$, we have $\mathcal{G}(c_i) \cap \mathcal{G}(c_j) = \emptyset$. For every $m \in \mathbbm{N}$, construct $\mathcal{H}_m$ on $C_m$ as in Lemma \ref{lem1}. In addition, a key part of this proof is to ensure that the hypothesis in $\mathcal{H}_m$ are non-robust to points in $C_{m'}$ for all $m' \neq m$. To do so, we will need to adjust each hypothesis $h_b \in \mathcal{H}_m$ carefully. By definition, for every $m \in \mathbbm{N}$,  $\mathcal{H}_m$ consists of $2^{3m}$ hypothesis of the form 
 $$h_b(z) = \begin{cases}
0 & \text{if $z \in \bigcup_{i=1}^{3m} \mathcal{B}_i^b$} \cup \mathcal{B}_i \\
1 & \text{otherwise}
\end{cases}$$
for each bitstring $b \in \{0, 1\}^{3m}$.
Note that the same set $\bigcup_{i=1}^{3m} \mathcal{B}_i$ is shared across every hypothesis $h_b \in \mathcal{H}_m$. For each $m \in \mathbbm{N}$, let $\mathcal{B}^m = \bigcup_{i=1}^{3m} \mathcal{B}_i$ be exactly the union of these $3m$ sets. Next, from the construction in Lemma \ref{lem1},  for every center $c_i \in C_m$, $\mu_{c_i}\left(\mathcal{B}_i \cup \left(\bigcup_{b} \mathcal{B}_i^b\right)\right) \leq \frac{1+\rho}{2} < 1$. Thus, there exists a set $\tilde{\mathcal{B}}_i \subset \G(c_i)$ s.t. $\mu_{c_i}(\tilde{\mathcal{B}}_i) > 0$ and $\tilde{\mathcal{B}}_i \cap \left(\mathcal{B}_i \cup \left(\bigcup_{b} \mathcal{B}_i^b\right)\right) = \emptyset $. Consider one such subset $\tilde{\mathcal{B}}_i$ from each of the $3m$ centers in $C_m$ and let $\tilde{\mathcal{B}}^m = \bigcup_{i=1}^{3m} \tilde{\mathcal{B}}_i$. Finally, make the following adjustment to each $h_b \in \mathcal{H}_m$,
 $$h_b(z) = \begin{cases}
0 & \text{if $z \in \bigcup_{i=1}^{3m} \mathcal{B}_i^b$} \cup \mathcal{B}_i \text{ or $z \in \mathcal{B}^{m'} \cup \tilde{\mathcal{B}}^{m'}$ for  $m' \neq m$}\\
1 & \text{otherwise}
\end{cases}$$
One can verify that every hypothesis in $\mathcal{H}_m$ has a non-robust region (i.e. $\mathcal{B}^{m'} \cup \tilde{\mathcal{B}}^{m'}$ for  $m' \neq m$) with mass strictly bigger than $\rho$ in every center in $C_{m'}$ for every $m' \neq m$. Thus, the hypotheses in $\mathcal{H}_m$ are non-robust to points in $C_{m'}$ for all $m' \neq m$. Finally, as we did in Lemma \ref{lem2}, for each $m$, we only keep the subset of hypothesis $\mathcal{H}'_m = \{h_b \in \mathcal{H}_m: \sum_{i=1}^{3m}b_i = m\}$. Note that for each $m \in \mathbbm{N}$, the hypothesis class $\mathcal{H}'_m$ behaves exactly like the hypothesis class from Lemma \ref{lem2} on $C_m$.

Let $\mathcal{H} := \bigcup_{m=1}^{\infty} \mathcal{H}'_m$ and $\mathcal{G}(C_m) := \bigcup_{i=1}^{3m} \mathcal{G}(c_i)$. Since we have modified the hypothesis class, we need to reprove that its VC dimension is still at most $1$. Consider two points $x_1, x_2 \in \mathcal{X}$. If either $x_1$ or $x_2$ is not in $\bigcup_{m=1}^{\infty} \G(C_m)$ and not in $\bigcup_{m=1}^{\infty} \mathcal{B}^m \cup \tilde{\mathcal{B}}^m$, then all hypothesis predict $x_1$ or $x_2$ as $1$. If both $x_1$ and $x_2$ are in $\mathcal{B}^m \cup \tilde{\mathcal{B}}^m$ for some $m \in \mathbbm{N}$, then:
\begin{itemize}
\itemsep0em 
\item if either $x_1$ or $x_2$ are in $\mathcal{B}^m$, every hypothesis in $\mathcal{H}$ labels either $x_1$ or $x_2$ as 0. 
\item if both $x_1$ and $x_2$ are in $\tilde{\mathcal{B}}^m$, we can only get the labeling $(1, 1)$ from hypotheses in $\mathcal{H}_m$ and the labelling $(0, 0)$ from the hypotheses in $\mathcal{H}_{m'}$ for $m' \neq m$.
\end{itemize}
In the case both $x_1$ and $x_2$ are in $\G(C_m) \setminus (\mathcal{B}^m \cup \tilde{\mathcal{B}}^m)$, then, they cannot be shattered by Lemma \ref{lem1}. In the case $x_1 \in \mathcal{B}^m \cup \tilde{\mathcal{B}}^m$ and $x_2 \in \G(C_m) \setminus (\mathcal{B}^m \cup \tilde{\mathcal{B}}^m)$:
\begin{itemize}
\itemsep0em 
\item if $x_1$ is in $\mathcal{B}^m$, every hypothesis in $\mathcal{H}$ labels $x_1$  as 0. 
\item if $x_1$ is in $\tilde{\mathcal{B}}^m$ then, we can never get the labelling $(0, 0)$.
\end{itemize}
If $x_1 \in \mathcal{B}^i \cup \tilde{\mathcal{B}}^i$ and $x_2 \in \mathcal{B}^j \cup \tilde{\mathcal{B}}^j$ for $i \neq j$, then: 
\begin{itemize}
\itemsep0em 
\item if either $x_1$ or $x_2$ are in $\mathcal{B}^i$ or $\mathcal{B}^j$ respectively, every hypothesis in $\mathcal{H}$ labels either $x_1$ or $x_2$ as 0. 
\item if both $x_1$ and $x_2$ are in $\tilde{\mathcal{B}}^i$ and $\tilde{\mathcal{B}}^j$ respectively, we can never get the labelling $(1, 1)$.
\end{itemize}
In the case $x_1 \in \mathcal{B}^i \cup \tilde{\mathcal{B}}^i$ and $x_2 \in \G(C_j) \setminus (\mathcal{B}^j \cup \tilde{\mathcal{B}}^j)$ for $j \neq i$, then we cannot obtain the labelling $(1, 0)$. If $x_1 \in \G(C_i) \setminus (\mathcal{B}^i \cup \tilde{\mathcal{B}}^i)$ and $x_2 \in \G(C_j) \setminus (\mathcal{B}^j \cup \tilde{\mathcal{B}}^j)$ for $i \neq j$, then we cannot obtain the labelling $(0, 0)$. Since we shown that for all possible $x_1$ and $x_2$, $\mathcal{H}$ cannot shatter them, $\text{VC}(\mathcal{H}) \leq 1$.

We now use the same reasoning in \cite{montasser2019vc}, to show that no proper learning rule works. By Lemma \ref{lem2}, for any proper learning rule $\mathcal{A}: (\mathcal{X} \times \mathcal{Y})^* \rightarrow \mathcal{H}$ and for any $m \in \mathbbm{N}$, we can construct a distribution $\mathcal{D}$ over $C_m$ (which has $3m$ points from $\mathcal{X}$) where there exists a hypothesis $h^* \in \mathcal{H}'_m$ that achieves $R^{\rho}_{\G, \mu}(h^{*};\mathcal{D}) = 0$, but with probability at least $1/7$ over $S \sim \mathcal{D}^m$, $R^{\rho}_{\G, \mu}(\mathcal{A}(S);\mathcal{D}) > 1/8$. Note that it suffices to only consider hypothesis in $\mathcal{H}'_m$ because, by construction, all hypothesis in $\mathcal{H}'_{m'}$ for $m' \neq m$ are not probabilistically robust on $C_m$, and thus always achieve loss $1$ on all points in $C_m$. Thus, rule $\mathcal{A}$ will do worse if it picks hypotheses from these classes. This shows that the sample complexity of properly probabilistically robustly PAC learning $\mathcal{H}$ is arbitrarily large, allowing us to conclude that $\mathcal{H}$ is not properly learnable. 
\end{proof}

\section{Proofs for Section \ref{sec:relaxedcomp}}
\label{app:relaxedcomp}

\subsection{Proof of Lemma \ref{lem:sand}} \label{app: sanduc}
\begin{proof} (of Lemma \ref{lem:sand})
Let $\mathcal{A}(S) = \inf_{h \in \mathcal{H}} \E_S\left[\ell_2(h, (x, y))\right]$. By uniform convergence of $\tilde{\ell}(h, (x, y))$, we have that for sample size $m = n(\frac{\epsilon}{2}, \frac{\delta}{2})$, with probability at least $1 - \frac{\delta}{2}$, over a sample $S \sim \mathcal{D}^m$, for every hypothesis $h \in \mathcal{H}$ simultaneously,
$$\mathbbm{E}_\mathcal{D}\left[\tilde{\ell}(h, (x, y))\right] \leq \hat{\E}_S\left[\tilde{\ell}(h, (x, y))\right] + \frac{\epsilon}{2}.$$
In particular, this implies that for $\hat{h} = \mathcal{A}(S)$, we have
$$\mathbbm{E}_\mathcal{D}\left[\tilde{\ell}(\hat{h}, (x, y))\right] \leq \hat{\E}_S\left[\tilde{\ell}(\hat{h}, (x, y))\right] + \frac{\epsilon}{2}.$$
Since, $\ell_1(h, (x, y)) \leq \tilde{\ell}(h, (x, y)) \leq \ell_2(h, (x, y))$, we have that 
$$\mathbbm{E}_\mathcal{D}\left[\ell_1(\hat{h}, (x, y))\right] \leq \hat{\E}_S\left[\ell_2(h^*, (x, y))\right] + \frac{\epsilon}{2}$$
 where $h^* = \inf_{h \in \mathcal{H}}\E_{\mathcal{D}}\left[\ell_2(h, (x, y))\right]$. It now remains to upper bound $\hat{\E}_S\left[\ell_2(h^*, (x, y))\right]$ with high probability. However, a standard Hoeffding bound tells us that with probability $1 - \frac{\delta}{2}$ over a sample $S$ of size $O(\frac{\ln(\frac{1}{\delta})}{\epsilon^2})$, $\hat{\E}_S\left[\ell_2(h^*, (x, y))\right] \leq \mathbbm{E}_{\mathcal{D}}\left[\ell_2(h^*, (x, y))\right] + \frac{\epsilon}{2}.$ Thus, by union bound, we get that with probability at least $1 - \delta$, $\mathbbm{E}_\mathcal{D}\left[\ell_1(\hat{h}, (x, y))\right] \leq \mathbbm{E}_{\mathcal{D}}\left[\ell_2(h^*, (x, y))\right] + \epsilon,$ using a sample of size $n(\epsilon/2, \delta/2) + O(\frac{\ln(\frac{1}{\delta})}{\epsilon^2})$.  \end{proof}

\subsection{Proof of Theorem 
\ref{thm:properprobrelax}}
\label{app:properprobrelax}

\begin{proof}(of Theorem \ref{thm:properprobrelax})
Fix  $0 \leq \rho^* < \rho < 1$ and let $\mathcal{H}$ be a hypothesis class with $\text{VC}(\mathcal{H}) = d$. Let $(\G, \mu)$ be an arbitrary adversary, $\mathcal{D}$ be an arbitrary distribution over $\mathcal{X} \times \mathcal{Y}$, and $S = \{(x_1, y_1), ..., (x_m, y_m)\}$ an i.i.d. sample of size $m$. Let $\mathcal{A}(S) = \text{PRERM}(S; \G, \rho^*)$. 

 By Lemma \ref{lem:sand}, it suffices to show that there exists a loss function $\ell(h, (x, y))$ s.t. $\ell_{\mathcal{G}, \mu}^{\rho}(h, (x, y)) \leq \ell(h, (x, y)) \leq \ell^{\rho^*}_{\mathcal{G}, \mu}(h, (x, y)))$ and $\ell(h, (x, y))$ enjoys the uniform convergence property with sample complexity $n = O\left( \frac{\frac{d}{(\rho - \rho^*)^2} \ln(\frac{1}{(\rho - \rho^*) \epsilon}) + \ln(\frac{1}{\delta})}{\epsilon^2} \right)$. Consider the probabilistically robust ramp loss: 

$$\ell^{\rho, \rho^*}_{\G, \mu}(h, (x, y)) = \min(1, \max(0, \frac{\mathbbm{P}_{g \sim \mu}\left[h(g(x)) \neq y \right] - \rho^*}{\rho - \rho^*} )). $$

Figure \ref{fig:ramp_loss} visually showcases how the probabilistic robust losses at $\rho$ and $\rho^*$ sandwich the probabilistic ramp loss at $\rho, \rho^*$. 

\begin{figure}
  \centering
  \includegraphics[width=10cm]{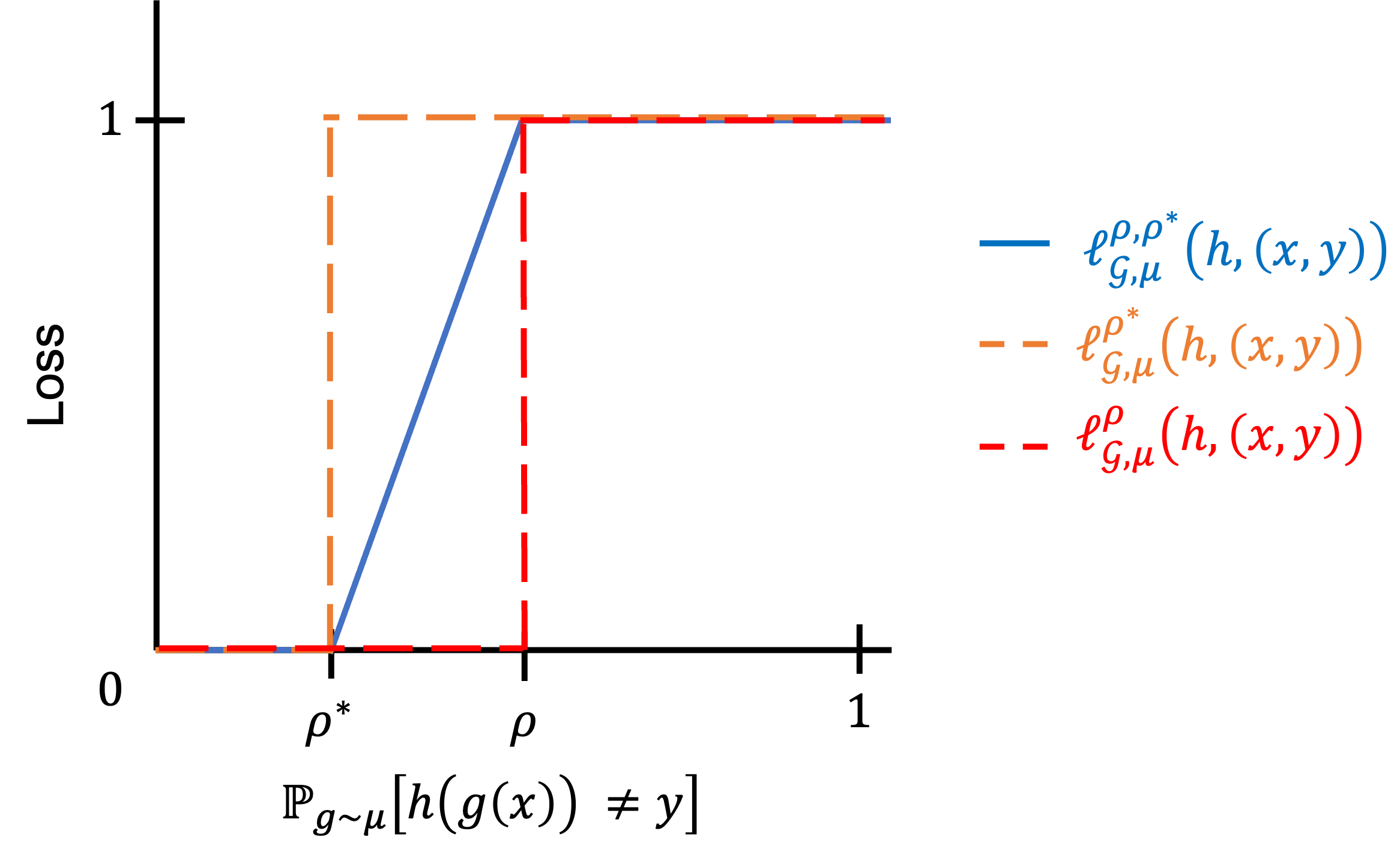}
  \caption{Comparison of probabilistic robust \textit{ramp} loss to probabilistic robust losses of hypothesis $h$ on example $(x, y)$. The probabilistic robust losses at $\rho$ and $\rho^*$ sandwich the probabilistic robust ramp loss at $\rho, \rho^*$.}
  \label{fig:ramp_loss}
\end{figure}

Its not too hard to see that $\ell_{\mathcal{G}, \mu}^{\rho}(h, (x, y)) \leq \ell^{\rho, \rho^*}_{\G, \mu}(h, (x, y)) \leq \ell^{\rho^*}_{\mathcal{G}, \mu}(h, (x, y)))$. Furthermore, since $\ell^{\rho, \rho^*}_{\G, \mu}(h, (x, y))$ is $O(\frac{1}{\rho - \rho^*})$-Lipschitz in $y\mathbbm{E}_{g \sim \mu}\left[h(g(x)) \neq y\right]$, by Lemma \ref{lem: uclip}, we have that $\ell^{\rho, \rho^*}_{\G, \mu}(h, (x, y))$ enjoys the uniform convergence property with sample complexity $O\left( \frac{\frac{d}{(\rho - \rho^*)^2} \ln(\frac{1}{(\rho - \rho^*) \epsilon}) + \ln(\frac{1}{\delta})}{\epsilon^2} \right)$. This completes the proof, as the conditions for Lemma \ref{lem:sand} have been met, and therefore the learning rule $\mathcal{A}(S) = \text{PRERM}(S; \G, \rho^*)$ enjoys the stated generalization guarantee with the specified sample complexity. 

\end{proof}

\subsection{Proof of Theorem \ref{thm:advrobrelax}}

\label{app:advroberelax}

\begin{proof}
(of Theorem \ref{thm:advrobrelax}) Fix  $0 < \rho$ and let $\mathcal{H}$ be a hypothesis class with $\text{VC}(\mathcal{H}) = d$. Let $\G$ be an arbitrary adversary, $\mathcal{D}$ be an arbitrary distribution over $\mathcal{X} \times \mathcal{Y}$, and $S = \{(x_1, y_1), ..., (x_m, y_m)\}$ an i.i.d. sample of size $m$. Let $\mathcal{A}(S) = \text{RERM}(S; \G)$. 

Fix a measure $\mu$ over $\mathcal{G}$. By Lemma \ref{lem:sand}, it suffices to show that there exists a loss function $\ell(h, (x, y))$ s.t. $\ell_{\mathcal{G}, \mu}^{\rho}(h, (x, y)) \leq \ell(h, (x, y)) \leq \ell_{\mathcal{G}}(h, (x, y)))$ and $\ell(h, (x, y))$ enjoys the uniform convergence property with sample complexity $n = O\left( \frac{\frac{d}{\rho^2} \ln(\frac{1}{\rho \epsilon}) + \ln(\frac{1}{\delta})}{\epsilon^2} \right)$. Consider the probabilistically robust ramp loss: 

$$\ell^{\rho, \rho^*}_{\G, \mu}(h, (x, y)) = \min(1, \max(0, \frac{\mathbbm{P}_{g \sim \mu}\left[h(g(x)) \neq y \right] - \rho^*}{\rho - \rho^*} )). $$

Letting $\rho^* = 0$, its not too hard to see that $\ell_{\mathcal{G}, \mu}^{\rho}(h, (x, y)) \leq \ell^{\rho, 0}_{\G, \mu}(h, (x, y)) \leq \ell_{\mathcal{G}}(h, (x, y)))$. Furthermore, since $\ell^{\rho, 0}_{\G, \mu}(h, (x, y))$ is $O(\frac{1}{\rho})$-Lipschitz in $y\mathbbm{E}_{g \sim \mu}\left[h(g(x)) \neq y\right]$, by Lemma \ref{lem: uclip}, we have that $\ell^{\rho, 0}_{\G, \mu}(h, (x, y))$ enjoys the uniform convergence property with sample complexity $O\left( \frac{\frac{d}{\rho^2} \ln(\frac{1}{\rho \epsilon}) + \ln(\frac{1}{\delta})}{\epsilon^2} \right)$. This completes the proof, as the conditions for Lemma \ref{lem:sand} have been met, and therefore the learning rule $\mathcal{A}(S)$ enjoys the stated generalization guarantee with the specified sample complexity. 
\end{proof}

\subsection{Proof of Theorem \ref{thm:tolerant}}
\label{app:tolerant}

\begin{proof} (of Theorem \ref{thm:tolerant}) Assume that there exists  a subset $\mathcal{G}' \subset \mathcal{G}$, that is $r$-Nice w.r.t. $\mathcal{H}$. By Lemma \ref{lem:sand}, it is sufficient to find a perturbation set $\tilde{\mathcal{G}}$ s.t. (1) $\ell_{\mathcal{G}'}(h, (x, y)) \leq \ell_{\tilde{\mathcal{G}}}(h, (x, y)) \leq \ell_{\mathcal{G}}(h, (x, y))$ and (2) $\ell_{\tilde{\mathcal{G}}}(h, (x, y))$ enjoys the uniform convergence property with sample complexity $O\left(\frac{\text{VC}(\mathcal{H})\log(\mathcal{N}_{r}(\mathcal{G}'_{2r}, d))\ln(\frac{1}{\epsilon})\ + \ln(\frac{1}{\delta})}{\epsilon^2}\right)$. Let $\tilde{\mathcal{G}} \subset \mathcal{G}$ be the minimal $r$-cover of $\mathcal{G}'_{2r}$ with cardinality $\mathcal{N}_{r}(\mathcal{G}'_{2r}, d)$.  By Lemma 1.1 of \cite{attias2021improved}, the loss class $\mathcal{L}_{\mathcal{H}}^{\tilde{\mathcal{G}}}$ has VC dimension at most $O(\text{VC}(\mathcal{H})\log(|\tilde{\mathcal{G}}|)) = O(\text{VC}(\mathcal{H})\log(\mathcal{N}_{r}(\mathcal{G}'_{2r})))$, implying that $\ell_{\tilde{\mathcal{G}}}(h, (x, y))$ enjoys the uniform convergence property with the previously stated sample complexity $O\left(\frac{\text{VC}(\mathcal{H})\log(\mathcal{N}_{r}(\mathcal{G}'_{2r}, d))\ln(\frac{1}{\epsilon})\ + \ln(\frac{1}{\delta})}{\epsilon^2}\right)$. Now, it remains to show that for our choice of $\tilde{\mathcal{G}}$, we have $\ell_{\mathcal{G}'}(h, (x, y)) \leq \ell_{\tilde{\mathcal{G}}}(h, (x, y)) \leq \ell_{\mathcal{G}}(h, (x, y))$. Since, $\tilde{\mathcal{G}} \subset \mathcal{G}$ ,the upperbound is trivial. Thus, we only focus on proving the lowerbound, $\ell_{\mathcal{G}'}(h, (x, y)) \leq \ell_{\tilde{\mathcal{G}}}(h, (x, y))$ for all $h \in \mathcal{H}$ and $(x, y) \in \mathcal{X} \times \mathcal{Y}$. Fix $h \in \mathcal{H}$ and $(x, y) \in \mathcal{X} \times \mathcal{Y}$. If $\ell_{\mathcal{G}'}(h, (x, y)) = 1$, then there exists a $g \in \mathcal{G}'$ s.t. $h(g(x)) \neq y$. Let $g$ denote one such perturbation function. By the $r$-Niceness property of $\mathcal{G'}$ w.r.t. $\mathcal{H}$, there must exist $B_r(g^*)$ centered at some $g^* \in \mathcal{G}$ such that $g \in B_r(g^*)$ and $h(g(x)) = h(g'(x))$ for all $g' \in B_r(g^*)$. This implies that $h(g'(x)) \neq y$ for all $g' \in B_r(g^*)$. Furthermore, since $B_{2r}(g)$ is the union of all balls of radius $r$ that contain $g$, we have that $B_r(g^*) \subset B_{2r}(g)$. From here, its not too hard to see that $B_r(g^*) \subset \mathcal{G}_{2r}'$ by definition. Finally, since $\tilde{\mathcal{G}}$ is an $r$-cover of $\mathcal{G}'_{2r}$, it must contain at least one function from $B_r(g^*)$. This completes the proof as we have shown that there exists a perturbation function $\hat{g} \in \tilde{\mathcal{G}}$ s.t. $h(\hat{g}(x)) \neq y$.
\end{proof}

\subsection{$\ell_p$ balls are $r$-Nice perturbation sets for linear classifiers}
In this section, we give a concrete example of a hypothesis class $\mathcal{H}$ and  metric space of perturbation functions $(\mathcal{G}, d)$ for which there exists an $r$-nice perturbation subset $\mathcal{G}' \subset \mathcal{G}$.  Let $\mathcal{X} = \mathbbm{R}^q$ and fix $r \in \mathbbm{R}_{\geq0}$. For the hypothesis class, consider the set of homogeneous halfspaces, $\mathcal{H} = \{h_w| w \in \mathbbm{R}^q\}$, where $h_w(x) = w^Tx$. Let $\hat{\mathcal{G}} = \{g_{\delta}: \delta \in \mathbbm{R}^q, ||\delta||_p \leq 3r\}$ where $g_{\delta}(x) = x + \delta$ for all $x \in \mathcal{X}$ and consider \textit{any} perturbation set $\mathcal{G}$ s.t. $\mathcal{G} \supset \hat{\mathcal{G}}$. That is, $\hat{\mathcal{G}}(x) = \{g(x): g \in \hat{\mathcal{G}}\}$ induces a $\ell_p$ ball of radius $3r$ around $x$. We will accordingly consider the distance metric $d(g_{\delta_1}, g_{\delta_2}) = \sup_{x \in \mathcal{X}} ||g_{\delta_1}(x) - g_{\delta_2}(x)||_p$. Restricted to the set $\hat{\mathcal{G}}$, this distance metric reduces to $d(g_{\delta_1}, g_{\delta_2}) = ||\delta_1 - \delta_2||_p = \ell_p(\delta_1, \delta_2)$ for $g_{\delta_1}, g_{\delta_2} \in  \hat{\mathcal{G}}$.  Finally, consider $\mathcal{G}' = \{g_{\tau}: \tau \in \mathbbm{R}^q, ||\tau||_p \leq r\} \subset \hat{\mathcal{G}} \subset \mathcal{G}$ which induces an $\ell_p$ ball of radius $r$ around $x$. 

We will now show that $\mathcal{G}'$ is $r$-nice perturbation set w.r.t $\mathcal{H}$. Let $x \in \mathcal{X}$, $h_w \in \mathcal{H}$, and $g_{\tau} \in \mathcal{G}'$. Let $c = h(g_{\tau}(x)) \in \{\pm1\}$. Consider the function $g_{\tau + \frac{crw}{||w||_p}}$. By definition, we have that $g_{\tau} \in B_r(g_{\tau + \frac{crw}{||w||_p}}) \subset \hat{\mathcal{G}} \subset \mathcal{G}$. To see this, observe that  $||\tau + \frac{crw}{||w||_p}||_p \leq 2r$ by the triangle inequality. Finally, it remains to show that for every $g' \in  B_r(g_{\tau + \frac{crw}{||w||_p}}) = \{g_{\tau + \frac{crw}{||w||_p} + \kappa}| \kappa \in \mathbbm{R}^d, ||\kappa||_p \leq r\}$, $h_w(g'(x)) = h_w(g_{\tau}(x)) = c$. Let $c = +1$ and consider the function $g'_{\tau + \frac{rw}{||w||_p} + \kappa} \in B_r(g_{\tau + \frac{rw}{||w||_p}})$. Note that  $w^{T}(x + \tau + \frac{rw}{||w||_p}+ \kappa) = w^T(x + \tau) + r||w||_p + w^T\kappa$. By Cauchy-Schwartz, we can lower bound  $w^T\kappa \geq -||w||_p||\kappa||_p \geq -r||w||_p$. Therefore, we have that $w^T(x +  \tau + \frac{rw}{||w||_p}+ \kappa) \geq w^T(x + \tau) > 0$, where the last inequality comes from the fact that $+1 = c = h_w(g_{\tau}) = \text{sign}(w^T(x + \tau))$. Therefore, $h(g'_{\tau + \frac{rw}{||w||_p} + \kappa}(x)) = \text{sign}(w^T(x +  \tau + \frac{rw}{||w||_p}+ \kappa)) =  \text{sign}(w^T(x + \tau)) = h(g_{\tau}(x))$ as desired. A similar proof holds when $c = -1$. Therefore, we have shown that $\mathcal{G}'$ is a $r$-nice perturbation set w.r.t $\mathcal{H}$.

We now can use Theorem \ref{thm:tolerant} to provide sample complexity guarantees on Tolerantly Robust PAC Learning with $\mathcal{G'}$ and $\mathcal{G}$. The main quantity of interest is $\log(\mathcal{N}_{r}(\mathcal{G}'_{2r}, d))$. However, note that $\mathcal{G}'_{2r} = \hat{\mathcal{G}}$. Therefore, we just need to compute $\log(\mathcal{N}_{r}(\hat{\mathcal{G}}, d)) = \log(\mathcal{N}_{r}(\{g_{\delta}: \delta \in \mathbbm{R}^q, ||\delta||_p \leq 3r\}, d))$. However, this is equal to  $\log(\mathcal{N}_{r}(\{\delta \in \mathbbm{R}^q : ||\delta||_p \leq 3r\}, \ell_p))$ using the $\ell_p$ distance metric since $g_{\delta}$ maps one-to-one to $\delta$. Using standard arguments, $\log(\mathcal{N}_{r}(\{\delta \in \mathbbm{R}^q : ||\delta||_p \leq 3r\}, \ell_p)) = \log(\mathcal{N}_{\frac{1}{3}}(\{\delta \in \mathbbm{R}^q : ||\delta||_p \leq 1\}, \ell_p))  = O(q)$ (\cite{bartlett2013}). Thus, overall, $\mathcal{H}$ is tolerantly PAC learnable w.r.t $(\mathcal{G}, \mathcal{G}')$ with sample complexity close to what one would require in the standard PAC setting.

\end{document}